\documentclass[opre,nonblindrev]{arxiv} 

\OneAndAHalfSpacedXI 

\usepackage{amsmath}
\usepackage{amsfonts}
\usepackage{dsfont}
\usepackage{enumerate}
\usepackage{algorithm}
\usepackage{bm}
\usepackage[noend]{algpseudocode}
\usepackage{fix-cm}
\usepackage{multirow}  
\usepackage{booktabs}
\usepackage{endnotes}
\let\footnote=\endnote
\let\enotesize=\normalsize
\def\notesname{Endnotes}
\def\makeenmark{\hbox to1.275em{\theenmark.\enskip\hss}}
\def\enoteformat{\rightskip0pt\leftskip0pt\parindent=1.275em
  \leavevmode\llap{\makeenmark}}

\usepackage{natbib}
 \bibpunct[, ]{(}{)}{,}{a}{}{,}
 \def\bibfont{\small}

\usepackage[colorlinks=true,breaklinks=true,bookmarks=false,urlcolor=blue,
     citecolor=blue,linkcolor=blue,bookmarksopen=false,draft=false]{hyperref}

\TheoremsNumberedThrough    
\ECRepeatTheorems

\EquationsNumberedThrough    
\usepackage{xcolor}

\begin{document}

\RUNAUTHOR{Chen and Al Kontar}

\RUNTITLE{Online Learning of Optimal Sequential Testing Policies}

\TITLE{Online Learning of Optimal Sequential Testing Policies}

\ARTICLEAUTHORS{
\AUTHOR{Qiyuan Chen, Raed Al Kontar}
\AFF{Industrial and Operations Engineering, \\University of Michigan, Ann Arbor, MI 48105, \\
\EMAIL{cqiyuan@umich.edu}, \EMAIL{alkontar@umich.edu}} 
} 

\ABSTRACT{
This paper studies an online learning problem that seeks optimal testing policies for a stream of subjects, each of whom can be evaluated through a sequence of candidate tests drawn from a common pool. We refer to this problem as the Online Testing Problem (OTP). Although conducting every candidate test for a subject provides more information, it is often preferable to select only a subset when tests are correlated and costly, and make decisions with partial information. If the joint distribution of test outcomes were known, the problem could be cast as a Markov Decision Process (MDP) and solved exactly. In practice, this distribution is unknown and must be learned online as subjects are tested. When a subject is not fully tested, the resulting missing data can bias estimates, making the problem fundamentally harder than standard episodic MDPs. We prove that the minimax regret must scale at least as $\Omega(T^{\frac{2}{3}})$, in contrast to the $\Theta(\sqrt{T})$ rate in episodic MDPs, revealing the difficulty introduced by missingness. This elevated lower bound is then matched by an Explore-Then-Commit algorithm whose cumulative regret is $\tilde{O}(T^{\frac{2}{3}})$ for both discrete and Gaussian distributions. To highlight the consequence of missingness-dependent rewards in OTP, we study a variant called the Online Cost-sensitive Maximum Entropy Sampling Problem, where rewards are independent of missing data. This structure enables an iterative-elimination algorithm that achieves $\tilde{O}(\sqrt{T})$ regret, breaking the $\Omega(T^{\frac{2}{3}})$ lower bound for OTP. Numerical results confirm our theory in both settings. Overall, this work deepens the understanding of the exploration–exploitation trade-off under missing data and guides the design of efficient sequential testing policies.
}

 \HISTORY{}

\maketitle

\section{Introduction}
Testing is a fundamental approach humans use to uncover latent information about the world. In many scenarios, the underlying information of interest, such as the quality of a product or the optimal treatment for a patient, is not directly observable. We refer to this problem as the Online Testing Problem (OTP). To obtain this latent information of interest, practitioners design a series of experiments, referred to as \textit{tests}, where each test reveals a glimpse of information related to the latent information of our interest. When possible, these tests can be conducted sequentially, allowing decisions about subsequent tests to be informed by the outcomes of prior ones. Ideally, conducting a sufficient number of relevant tests can fully reveal the desired latent information, while, in practice, tests often incur costs and exhibit correlations, making exhaustive testing inefficient. Thus, it is essential to balance the value of the information gained from a test against its associated cost. An optimal testing strategy would prioritize tests that offer unique and critical insights while minimizing redundancy and expenses. In cases where test correlations are known, scheduling tests efficiently is possible. However, in many real-world applications, these correlations are unknown and must be learned during operations. This paper addresses this more realistic scenario, formulating it as the \textbf{online testing problem} (OTP).

\subsection{Motivating Applications}
In this section, we introduce some practical applications that motivate the OTP. 

\subsubsection{Quality Control}
Testing is prevalent in quality control (QC). Imagine a bolt manufacturer who needs to assess whether the bolts they produce adhere to quality standards. From the standpoint of an engineer, there are numerous tests available, each providing an incomplete glimpse into a bolt's quality at a certain cost. Despite the long list of candidate tests, most outcomes are correlated: the length of a bolt, for instance, might be positively correlated with its weight, just as its shear strength is tied to its tensile strength \citep{budynas2011shigley}. Each of these tests varies not only in cost but also in fidelity. For example, while fluorescent penetrant inspection (FPI), ultrasonic testing (UT), and magnetic particle inspection (MPI) all serve to detect mechanical defects, each has distinct strengths: while UT has the deepest detection range down the surface of materials, FPI and MPI are generally available at a more affordable cost.

\begin{figure}[htb]
    \centering
    \includegraphics[width=\linewidth]{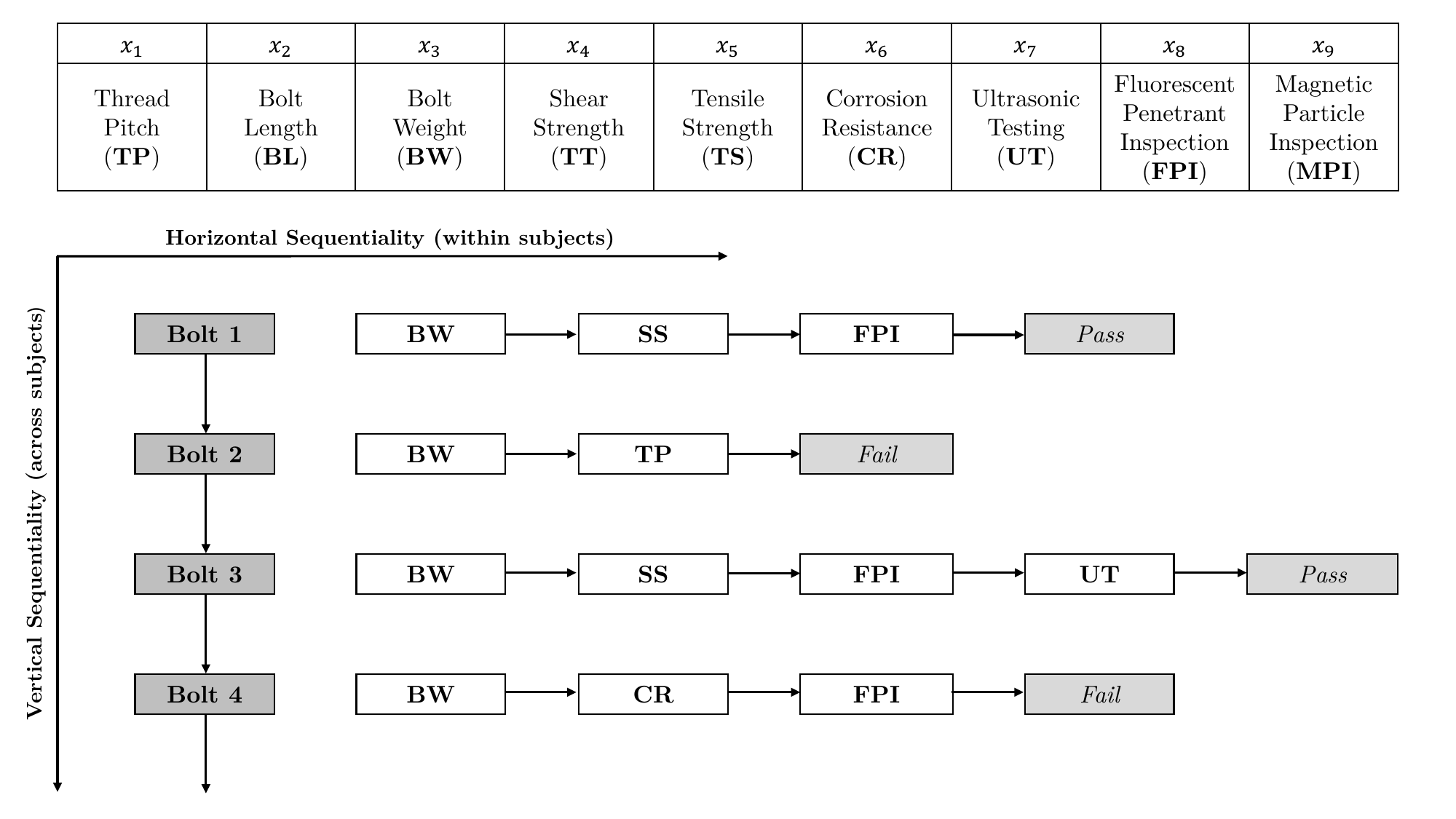}
    \caption{An example of a bolt manufacturer conducting sequential tests for quality control (QC)}
    \label{fig:inspection}
\end{figure}

\begin{table}[htb]
\centering
  \begin{tabular}{cccccccccc}
    \toprule
          & \multicolumn{9}{c}{\textbf{Tests}} \\         
    \cmidrule(lr){2-10}                                   
    \textbf{Subject} & \textbf{TP} & \textbf{BL} & \textbf{BW} & \textbf{SS} &
      \textbf{TS} & \textbf{CR} & \textbf{UT} & \textbf{FPI} & \textbf{MPI} \\  
    \midrule
    \textbf{Bolt 1} & \texttt{NA} & \texttt{NA} & 3 & 5 & \texttt{NA} & \texttt{NA} & \texttt{NA} & 8 & \texttt{NA} \\
    \textbf{Bolt 2} & 2 & \texttt{NA} & 3 & \texttt{NA} & \texttt{NA} & \texttt{NA} & \texttt{NA} & \texttt{NA} & \texttt{NA} \\
    \textbf{Bolt 3} & \texttt{NA} & \texttt{NA} & 6 & 5 & \texttt{NA} & \texttt{NA} & 3 & 4 & \texttt{NA} \\
    \textbf{Bolt 4} & \texttt{NA} & \texttt{NA} & 2 & \texttt{NA} & \texttt{NA} & 8 & \texttt{NA} & 8 & \texttt{NA} \\
    $\vdots$ & $\vdots$ & $\vdots$ & $\vdots$ & $\vdots$ & $\vdots$ & $\vdots$ &
      $\vdots$ & $\vdots$ & $\vdots$ \\
    \bottomrule
  \end{tabular}
\vspace{1em}
\caption{An anecdotal dataset collected during the QC process in Figure \ref{fig:inspection}}
\label{tab:dataset}
\end{table}

We shall highlight that there are two layers of sequentiality in the OTP. Ideally, if the correlations between the tests are known, instead of conducting every test, it makes sense to leverage these correlations to save test costs. There are two perspectives on how the correlation can help. From one perspective, some tests could be saved based on the outcome of another test: for example, if the shear strength of a bolt is strong, then one might consider skipping the tensile strength test. From another perspective, some tests might not be worth their costs and can be eliminated: if a bolt passes FPI, then the cost of UT may not even pay for the reduction in defect rate. These two perspectives together consist of the \textbf{horizontal sequentiality} in OTP (see Figure \ref{fig:inspection}): choosing the optimal next test based on previous test results and terminating tests when information is sufficient. In practice, because such correlation is usually unclear in practice, it has to be learned by observing the test outcomes on different bolts, where comes the \textbf{vertical sequentiality}. 

\subsubsection{Medical Diagnosis}
\label{sec:medical}
The above example of sequential testing with correlated tests also appears naturally in healthcare applications. In the vast realm of medical testing, there are numerous diagnostic tools available, each providing a partial snapshot of the patient’s health at a certain cost and level of invasiveness. Despite the extensive list of potential tests, many of their results are interrelated. Doctors must, therefore, strategically sequence these tests, choosing those most likely to provide critical information while minimizing unnecessary expense, invasiveness, and patient burden. For example, in the diagnosis of gout, useful tools range from blood panels to synovial fluid analysis and dual-energy CT scans. Each test differs not only in cost but also in diagnostic precision. Although a dual-energy CT scan is considered a gold standard in the diagnosis of gout \citep{zhang2006eular}, it is indeed costly compared to its alternatives and can be partially predicted by other clinical and laboratory results that are cheaper to obtain, including examples of Creatinine clearance and monosodium urate (MSU) crystals in synovial fluid \citep{gamala2018gouty}. Therefore, the optimal decision varies across different subjects, and doctors must balance the risks and costs based on every subject's symptoms and outcomes of lab results, which corresponds to the horizontal sequentiality in our setting. The qualities of such decisions, obviously, are based on the doctors' understanding of the correlations of different tests. Apart from learning from statistical studies like \cite{gamala2018gouty}, doctors can also learn the correlation while seeing more patients during clinical practice, which corresponds to the vertical sequentiality in our setting.

\subsubsection{Environmental Monitoring}
The application of sequential testing is not restricted only to different kinds of tests, but it can be the same test in different geographical or temporal contexts. One example of this kind is in environmental monitoring. For example, to monitor key environmental metrics (e.g., air quality, temperature, precipitation) in a country, researchers deploy sensors across different locations \citep{zidek2002designing}. Naturally, the information collected by different sensors is spatially correlated, while the operation of each sensor is associated with a cost. If one can learn the correlations, then it is possible to implement a more efficient solution where one maximizes the information gain minus the operation costs of the sensors. Interestingly (and jumping ahead a bit to Section \ref{sec:onlineMESP}), in cases where the goal is solely to maximize information gain, the reward becomes independent of the missing tests. This structural simplification makes the problem easier (in terms of regret) than the general OTP presented in Section \ref{sec:general}.

Across all these applications, the central objective remains the same: to maximize the reward of final decisions while minimizing the cumulative cost of testing. Achieving this goal is only possible if one understands how test outcomes are correlated, since correlated tests offer diminishing informational returns. Despite the important role correlation plays, it is often unknown a priori in practical settings. Even when the direction of dependence is understood, the exact magnitudes are typically unknown and must be inferred from data. In this paper, we formulate this sequential testing problem as an online learning problem with the hope of learning the optimal testing policy during operations.  Hereon, we denote the units tested (bolts, patients,..) as \textit{subjects}, and the decision makers as \textit{agents}.

\subsection{Brief Problem Description}
\label{sec:general}
Consider an online decision-making problem where an agent conducts sequential tests to probe the optimal decision for a subject indexed by $t\in [T]$. We model the desired latent information as the optimal decision in this decision-making problem. For instance, in the QC example in Figure \ref{fig:inspection}, every $t$ corresponds to a bolt to be examined, and the decisions are classifying bolts as either defective or qualified. For every subject $t$, a vector $\bm{x}^{(t)} \in \mathcal{X} \subset \mathbb{R}^d$ collects the hidden test outcomes as its entries, where each entry is revealed only when the corresponding test is conducted. These tests are designed to identify the optimal decision $y^{(t)} \in \mathcal{Y}$ for the corresponding subject. For the scope of this paper, we assume the environment is stationary and stochastic. In other words, $\bm{x}^{(t)}$ are independent and identically distributed random vectors following a stationary distribution $\mathcal{P}$. If $\mathcal{P}$ is a discrete distribution, we denote $|\mathcal{P}|$ as the size of its support. 
 
To reveal the $i$-th element $x^{(t)}_i$ of a random vector $\bm{x}^{(t)}$, the agent pays a test cost of $c_i$ to conduct the corresponding test. In our model, we assume the observations are noiseless, which do not require repetitions. 

If tests were free, one would conduct all of them, as more information generally leads to better decisions. However, in practice, tests are costly and test outcomes are often highly correlated, reducing the marginal benefits of additional tests. Therefore, it is crucial to strike a balance and sometimes make decisions with incomplete information when the marginal reward gained from the marginal information does not pay for its cost.

To keep the problem non-trivial, we assume the optimal decision of our interest is identifiable with the collection of tests we have. In other words, if we exhaust all the tests (reveal the entire $\bm{x}^{(t)}$), we would know the reward one can take by making a specific decision $y$. We denote this reward function as $f:\mathcal{X} \times \mathcal{Y} \to \mathbb{R}$, which is given to the agent a Priori. As such, the optimal decision for subject $t$ is the one that maximizes the reward, i.e., $y^{(t)}=\max_y f(\bm{x}^{(t)},y^{(t)})$. In the bolt QC example,  $-f(\bm{x}^{(t)},\texttt{fail})$ measures the manufacturing cost of the bolt, and $-f(\bm{x}^{(t)},\texttt{pass})$ measures the risk of selling a part (e.g., goodwill impairment) with quality test outcomes being $\bm{x}^{(t)}$. We highlight that this function $f$ takes only complete vectors as inputs, so the cost for decisions made under partial information (when $\bm{x}^{(t)}$ has missing entries, see Table \ref{fig:inspection}) is unknown.

The goal of the sequential decisions is to achieve a higher expected reward. For every subject $t$, there are two types of decisions that one can take: either conduct a test, where the agent observes the test outcome and pays a cost, or terminate testing and make a decision $y\in \mathcal{Y}$ based on the existing test outcomes. As an example shown in Figure \ref{fig:inspection}, for every bolt, the sequence of decisions happens horizontally from left to right: the agent either chooses a test (labeled in white boxes) or terminates testing and can make a decision of \textit{pass} or \textit{fail} (labeled in grey boxes). Denote the set $S$ as the set of all the tests that were conducted before the final decision. Then, the final reward for subject $t$ is measured as $f(\bm{x}^{(t)},y)-\sum_{i\in S} c_i$. 

Importantly, despite the objective existence of the final reward $f(\bm{x}^{(t)},y)-\sum_{i\in S} c_i$, it is not directly observed when an agent only has an incomplete $\bm{x}^{(t)}$. This becomes a central distinction that drives the OTP away from the conventional episodic MDP framework. In a standard episodic MDP \citep{azar2017minimax}, all the actions, states, and their associated rewards are observed.

\subsection{Related Literature}
\label{sec:review}

Sequential design of experiments has a long history since \cite{wald1948optimum} and \cite{chernoff1959sequential} and still remains an active and impactful research area. It captures a common practice in scientific discoveries: scientists often design experiments sequentially based on the outcomes of the previous experiments and terminate tests once there is enough evidence. 

Sequential design problems can adopt different problem setups depending on the type of experiments and their outcomes. Most of the existing literature on the sequential design of experiments is studied under sequential hypothesis testing \citep[SHT,][]{naghshvar2013active,bibaut2022near}.  Although we position this work under the umbrella of the sequential design of experiments, we are the first to consider learning a testing policy across different subjects online with no prior information, where the existing literature in SHT has little to offer. 

There exist two main streams of research in SHT operating under two different sets of assumptions. One line of research \citep{steven2021time-uniform,steven2022sequential,bibaut2022near} assumes doing a single experiment multiple times, and the goal is to find the optimal number of repetitions on the fly. This setting is most suited for A/B testing problems where there is only one experiment to do. The goal is to identify the optimal number of repetitions for a given experiment. This is fundamentally different than our setting because we have multiple different tests, and our goal is to find what is the best next test. Another line of research \citep{naghshvar2013active} is closer to our setting as it also involves multiple tests to choose from. However, their goal is still to test a single subject, and there is no learning involved since all transition probabilities between states are known beforehand.  

Our OTP, as its name suggests, is categorized as an online learning problem. Bandits \citep{lattimore2020bandit} are one of the most fundamental problems in online learning. Our problem is closer to partial monitoring \citep{cesa2006regret}, which can be viewed as a bandit problem with limited feedback. However, both bandit and partial monitoring usually do not plan for future decisions. The framework closest to our work in online learning for sequential decision-making is known as the Episodic MDP \citep{azar2017minimax,zanette2019tighter,domingues2021episodic}. Episodic MDP is a well-studied problem with a mini-max cumulative regret of $\tilde{\Theta}(\sqrt{T})$ (showing only $T$ dependence here). In Section \ref{sec:missingdata}, we will show that OTP is fundamentally different than episodic MDP due to the effect of missingness. If one were to force this problem into the framework of an Episodic MDP, OTP would have missing reward feedback in particular steps. The only way to approximate this missing reward is to rely on the estimated transition kernel, whose estimation accuracy is itself dependent on the past data. Consequently, there is a direct conflict between exploration and exploitation, as will be shown in the following sections, making the OTP fundamentally harder as the mini-max regret must scale with $\tilde{\Omega}(T^{\frac{2}{3}})$. 

Although we are the \textit{first} to study the testing problem in an online setting, prior applied work by \citet{peng2018} and \citet{yu2023deep} have trained deep reinforcement learning models for automated medical diagnosis in an offline setting (see the example in Section  \ref{sec:medical}). Interestingly, our analysis provides a theoretical justification for their approach of decoupling data collection from model training. Moreover, our findings highlight a critical concern for practitioners: missing data, if not properly addressed, can introduce bias and lead to suboptimal decision policies.

We acknowledge that there is another line of work that shares the setting of \cite{naghshvar2013active} but aims mainly at reducing computational complexity for testing problems \citep{kodialam2001throughput,condon2009algorithms,hellerstein2017max, gan2021greedy,segev2022poly}. Although many testing problems can be formulated as an MDP that can be solved in polynomial time with respect to the number of states, the number of states is usually exponential with respect to the number of tests without special structures. Researchers along this line assume a known data-generating distribution, and their main goal is to develop approximation algorithms that run in $O(\text{poly}(d))$ time. That said, the scope of this paper focuses on the learnability of the OTP rather than its computational complexity. In problems of high dimensions, one could use deep reinforcement learning \citep{yu2023deep} for an approximate solution. 

\subsection{Main Results and Our Contributions}

\label{sec:results}

This paper studies the problem of learning the optimal policy for sequentially selecting tests in an online setting (the Online Testing Problem (OTP)). The focus of this paper is to investigate theoretically what missing data brings to the learning of such a sequential decision-making problem. We demonstrate that the inherent data missingness makes the learning task provably more difficult than standard episodic MDPs, resulting in an elevated regret lower bound of \(\Omega(T^{\frac{2}{3}})\) (Theorem \ref{thm:lb_single}). In response, we propose an algorithm (Algorithm \ref{alg:etcD}) that matches this rate (up to logarithmic factors) where it achieves \(\tilde{O}(d|\mathcal{P}|^{\frac{1}{3}}T^{\frac{2}{3}})\) for discrete distributions (Theorem \ref{thm:etcD}) and \(\tilde{O}(d\sigma^2T^{\frac{2}{3}})\) for Gaussian distributions (Theorem \ref{thm:etcG}).

We then introduce the online cost-sensitive maximum entropy sampling problem (OCMESP), a special case of our general framework. OCMESP is provably more tractable, with a regret upper bound of \(\tilde{O}(d^3\sigma\sqrt{T})\) (Theorem \ref{thm:itr}) under our proposed Algorithm \ref{alg:itr}. All main theoretical results are validated through numerical simulations.

We shall summarize our key contributions as follows:

1. \textbf{Fundamental challenge of missing data:} We pinpoint a core challenge in OTP: the missing data problem. We show that the missingness pattern is Missing At Random \citep[MAR,][]{little2019statistical}, which can introduce bias if not properly addressed. To theoretically capture the impact of this missingness, we construct a lower bound example that reduces the OTP to a bandit setting where the reward of pulling one arm can only be observed by pulling another. This result quantifies the effect of missing data in OTP and highlights the importance of carefully designing the data collection process.

2. \textbf{Performance Guarantees of ETC:} We show that the explore-then-commit (ETC) strategy can match the optimal $T$ dependence, thereby providing a theoretical foundation for decoupling exploration and exploitation. This result supports current practices of training models on offline datasets \citep{peng2018,yu2023deep}, suggesting that such methods can be readily adapted to online settings. We provide convergence guarantees in both discrete and Gaussian settings, which may be of independent interest. Notably, although we adopt an ETC framework, our algorithm requires estimating rewards from transition probabilities, making it fundamentally different from existing theoretical results for standard episodic MDPs.

3. \textbf{Technical contribution in ETC regret upper bound:} While the dependence on $O(T^{\frac{2}{3}})$ is expected for ETC algorithms, the main challenge lies in the dependence of $|\mathcal{P}|$ and $d$. To the best of our knowledge, there exist no results for ETC algorithms in even episodic MDP literature. Instead of considering the dependence on the state and action space, we examine the dependence on the support size $|\mathcal{P}|$ and dimension $d$, which ends up giving a sharper bound, especially when $|\mathcal{P}|$ is sparse. Also, because we do not directly estimate transition probabilities for state-action pairs, the proof does not follow the standard literature in episodic MDP but uses an error propagation argument through dynamic programming. This technique prevents introducing extra dependence on $d$ by combining active actions space and state space sizes into a single $|\mathcal{P}|$.

4. \textbf{A practical setting avoids the general lower bound:} Recognizing the general problem’s inherent difficulty, we show that not all instances balancing information gain and testing costs suffer from the \(\Omega(T^{\frac{2}{3}})\) lower bound. If rewards are independent of missing entries, it is possible to design algorithms with \(\tilde{O}(\sqrt{T})\) regret. To illustrate this, we propose a practical special case, OCMESP, which achieves a regret bound of \(\tilde{O}(d^3\sigma\sqrt{T})\).

\section{The General Online Testing Problem} \label{sec:2}

\subsection{MDP Formulation}
\begin{figure}
    \centering
    \includegraphics[width=\linewidth]{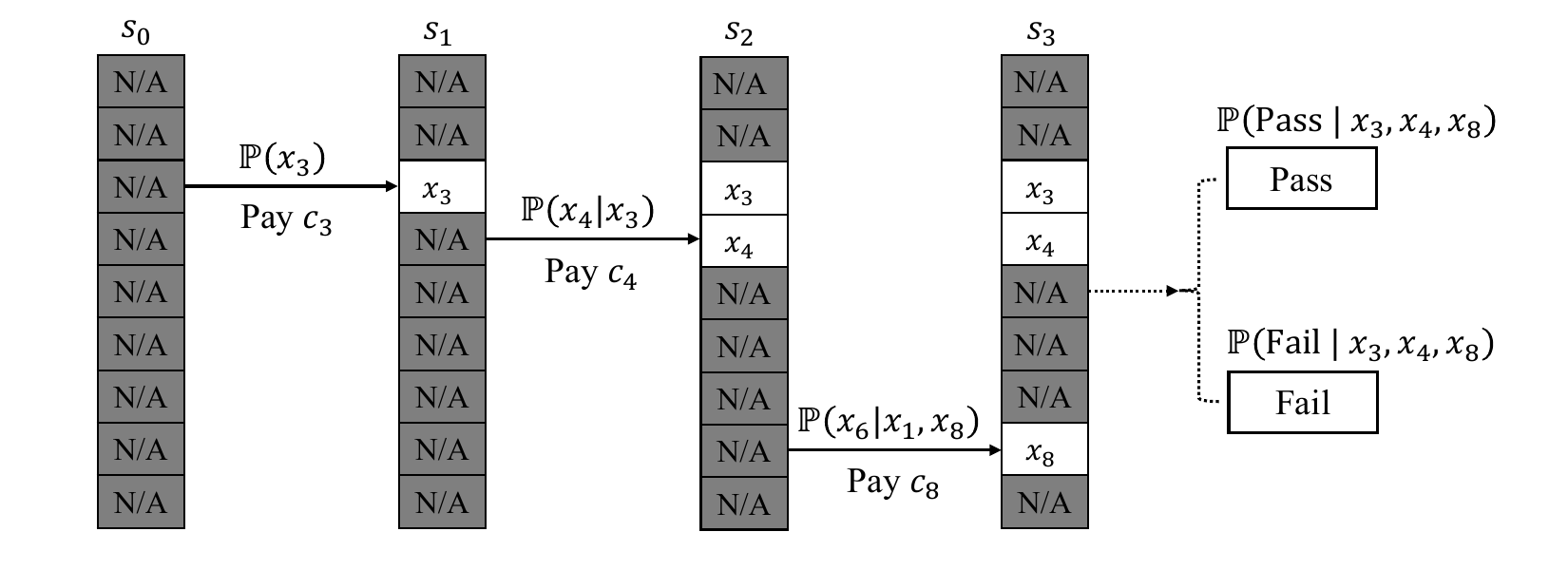}
    \caption{Formulation of the Online Testing Problem (Based on the Example in Figure  \ref{fig:inspection})}
    \label{fig:mdp}
\end{figure}

We formulate the OTP in the language of an MDP. Consider an $T$-episode MDP defined by a $4$-tuple $(\mathcal{S},\mathcal{A},q,r)$ where $\mathcal{S}$ is the state space, $\mathcal{A}$ is the action space, $q$ is the transition kernel, and $r$ is the reward function. At a high level, every episode corresponds to a subject, and the goal of an agent is to perform a sequence of tests until the agent makes a final decision. The action space $\mathcal{A}$ is defined as $\mathcal{A}\triangleq [d]\cup \mathcal{Y}$, where we use $[d]$ to denote the set of tests $\{1,\ldots, d\}$ and $\mathcal{Y}$ is the set of final decisions to be made. At every step within one episode, the agent can either continue testing by choosing an unperformed test in $[d]$ or terminate testing and choose the final decision from $\mathcal{Y}$. The states are represented by a $d$-dimensional vector $\bm{s}$ that keeps track of the test outcomes of a subject, which is simply $\bm{x}$ with its unobserved entries masked by $\texttt{NA}$. Upon making the final decision, an episode terminates to state $\texttt{END}$. As such, $\mathcal{S}\triangleq \left\{[s_1,\dots,s_d] \mid s_i\in \mathbb{R}\cup \{\texttt{NA}\}\right\} \cup \{\texttt{END}\}$. An example is shown in Figure \ref{fig:mdp}.

The transition probabilities of the MDP are given by $q$, where taking action $a$ in state $\bm{s}$ results in a distribution of the next state $\bm{s}' \sim q(\cdot \mid \bm{s},a)$.  For notation purposes, we define operation \(\bm{s} \oplus (a,b)\) as generating a new state by replacing the $a$-th element of $\bm{s}$ with value $b$, i.e.,
\[\bm{s} \oplus (a,b) \triangleq   \bm{s}' , \quad \text{where } s'_i = \begin{cases}
    b, & \text{ if $i=a$}\\
    s_i, & \text{otherwise}
\end{cases}.\]
Due to the definition of our states and actions, given action $a\in [d]$ in a state $\bm{s}$, the next state will be $\bm{s}\oplus (a,x_a)$, which is a random variable following the posterior distribution of the outcome of test $a$ conditioning on observed entries, i.e.,
\(q( \bm{s}\oplus (a,b) \mid \bm{s},a) = \mathbb{P}(x_a = b \mid x_j = s_j \text{ for all } j \text{ such that } s_j \ne \texttt{NA}),\) where $\mathbb{P}$ denotes the probability of an event. To simplify notations, we denote $P^{\bm{s}}(\cdot) = \mathbb{P}(\cdot \mid \bm{x} \overset{\Re}{=} \bm{s})$, where
\(\bm{x} \overset{\Re}{=} \bm{s} \Longleftrightarrow x_j = s_j \text{ for all } j \text{ such that } s_j \ne \texttt{NA}.\) Intuitively, $P^{\bm{s}}$ is the posterior distribution conditioning on current information $\bm{s}$. If the support of $\mathcal{P}$ is continuous, we abuse the notation by letting $P^{\bm{s}}(\cdot) = p(\cdot \mid \bm{x} \overset{\Re}{=} \bm{s})$, where $p$ is a probability density function.

When a test is chosen (i.e., $a\in [d]$), it incurs a corresponding cost of $c_a$, so the reward is given as $r(\bm{s},a) = - c_a$. Otherwise, if the agent makes a decision $y\in \mathcal{Y}$, an episode is terminated. The reward is determined by the quality of the decision, which is a realization of the posterior distribution conditioning on observations $\bm{s}$ (i.e., $r(\bm{s},y) \sim f(\bm{x},y)\mid \bm{s}$). 

An admissible testing policy $\pi$ is a mapping from states to actions, which could be stochastic. The entire purpose of the problem is to learn the optimal policy $\pi^*$ that maximizes the expected reward of an episode, where the expectation is over the randomness of the test outcome vector $\bm{x}$.

Note that the horizon and the initial state distribution are suppressed for simplicity. The horizon can be suppressed because any non-trivial policy terminates within $d$ steps, which is no worse than doing all the tests, so the horizon is implicitly capped at $d$. Without loss of generality, we do not consider the initial state distribution in this paper. We assume all the subjects arrive without contexts (i.e., no test outcome revealed), but our analysis can be adapted to contextual settings. 

\paragraph{Asymptotic Notations}
For asymptotic notations, we use $a = O(b)$ to denote that $a$ is asymptotically less than $b$ and $a = \Omega(b)$ to denote that $a$ is asymptotically greater than $b$. We use $\, \tilde{} \,$ to suppress any parameter-independent poly-log dependence, and we use $a\lesssim b$ to denote $a = \tilde{O}(b)$.

\subsection{Clairvoyant Policy}
\label{sec:clair}
Let $\mathcal{P}$ represent the underlying data-generating distribution of the test outcome vector $\bm{x}$. If $\mathcal{P}$ is known, a clairvoyant agent has full knowledge of the transition probabilities and rewards, allowing it to compute an optimal policy using dynamic programming (DP) shown in Algorithm \ref{alg:dp}. We denote this clairvoyant optimal solution as $\pi^*$. We use $v(\bm{s})$ to denote the value function, which is the expected remaining reward at state $\bm{s}$ when operating policy $\pi$ under data generating distribution $\mathcal{P}$. 

As mentioned above, for any state $\bm{s}$, a policy can choose to either do a test or make a decision. 
If the policy chooses to do test $i$, then it pays a cost of $c_i$ and has a probability of $P^{\bm{s}}({\bm{s}}\oplus (i,x_i))$ going to state ${\bm{s}}\oplus (i,x_i)$.
By the DP updating rule, we can compute the expected remaining reward for taking $i$ at state ${\bm{s}}$ as 
\[Q({\bm{s}},i) = -c_{i}+ \sum_{x_i} v( {\bm{s}}\oplus (i,x_i))\cdot P^{\bm{s}}( x_i) \quad \text{(discrete $\bm{x}$)},\]
\begin{equation}
\label{eq:qvalue}
   Q({\bm{s}},i) = -c_{i}+ \int v( {\bm{s}}\oplus (i,x_i))\cdot P^{\bm{s}}( x_i) \,dx_i \quad \text{(continuous $\bm{x}$)},
\end{equation}
so the optimal test at state ${\bm{s}}$ is $\argmax_i Q({\bm{s}},i)$, and the value function $v( {\bm{s}}) = \max_i Q({\bm{s}},i)$. Alternatively, if a decision $y$ was made in state ${\bm{s}}$, the expected reward $\mathbb{E}[r({\bm{s}},y)]$ can be computed as 
\[\mathbb{E}[r({\bm{s}},y)] = \sum_{\bm{x}} f(\bm{x},y)\cdot P^{\bm{s}}(\bm{x}), \quad \text{(discrete $\bm{x}$)}\]
\begin{equation}
\label{eq:rdecision}
    \mathbb{E}[r({\bm{s}},y)] = \int f(\bm{x},y)\cdot P^{\bm{s}}(\bm{x})\,d \bm{x} \quad \text{(continuous $\bm{x}$)},
\end{equation}
so the optimal decision at state ${\bm{s}}$ is $\argmax_y \mathbb{E}[r({\bm{s}},y)]$. If all the tests are done, there remains no uncertainty in the subject (i.e., $\bm{s}=\bm{x}$), so the optimal action is obviously to make a decision that maximizes $f(\bm{x},y)$. This gives the boundary conditions for our DP. As such, for ${\bm{s}}$ that does not contain missing values, the optimal decision and the value function are, respectively,
\(\pi^*({\bm{s}}) = \argmax_y f(\bm{s},y)\) and \(v({\bm{s}}) = \max_y f(\bm{s},y).\)

If there are still missing values in ${\bm{s}}$, then a policy can choose either to do a test that has not been done yet or make a decision directly. This only requires comparing $\max_y \mathbb{E}[r({\bm{s}},y)]$ and $\max_i Q({\bm{s}},i)$ and picking the one that is larger. This algorithm process is summarized in Algorithm \ref{alg:dp}.

\begin{algorithm}
\caption{Dynamic Programming (DP)}
\label{alg:dp}
\begin{algorithmic}[1]
\State \textbf{Input:} Distribution $\mathcal{P}$
\For{$\bm{s}$ being states that has no \texttt{NA}}
\State Compute $\max_y \mathbb{E}[r(\bm{s},y)]$ using Equation \ref{eq:rdecision}
\State Set $\pi^*(\bm{s})=\argmax_y \mathbb{E}[r(\bm{s},y)], v(\bm{s}) = \max_y \mathbb{E}[r(\bm{s},y)]$
\EndFor
\For{$d' = 1,\dots,d$}
\For{$\bm{s}$ being states that has $d'$ number \texttt{NA}}
\State Compute $\max Q(\bm{s},a)$ using Equation \ref{eq:qvalue}
\State Compute $\max_y \mathbb{E}[r(\bm{s},y)]$ using Equation \ref{eq:rdecision}
\If{$\max Q(\bm{s},a)>\max_y \mathbb{E}[r(\bm{s},y)]$}
\State Set $\pi^*(\bm{s})=\argmax Q(\bm{s},a), v(\bm{s}) = \max Q(\bm{s},a)$
\Else 
\State Set $\pi^*(\bm{s})=\argmax_y \mathbb{E}[r(\bm{s},y)], v(\bm{s}) = \max_y \mathbb{E}[r(\bm{s},y)]$
\EndIf
\EndFor
\EndFor
\State \textbf{Output}: $\pi^*$
\end{algorithmic}
\end{algorithm}

In practice, the data-generating distribution $\mathcal{P}$ is usually not available to us and can only be estimated from past data. One straightforward remedy is to replace the distribution $\mathcal{P}$ with a distribution $\hat{\mathcal{P}}$ that is estimated from the data. An estimated policy $\pi^*$ is then computed using the estimated distribution.  Alternatively, rather than obtaining the policy in two steps, existing literature \citep{peng2018,yu2023deep} directly trains reinforcement learning models on offline datasets, although they lose guarantees in optimality. 

The performance of any policy $\pi$ is measured as its gap compared to the optimal policy $\pi^*$. We note that the clairvoyant only knows the distribution of the $\bm{x}$, but not the realization of $\bm{x}$, and the performance gap (also known as the simple regret) $\ell(\bm{x},\pi)$ is measured as the gap between the rewards accumulated by $\pi$ and $\pi^*$ on subject $\bm{x}$.

Intuitively, with a sufficient amount of data, the hope is that the estimated distribution converges to the true distribution (i.e., $\hat{\mathcal{P}}\to\mathcal{P}$), which gives a better estimation of the value functions (i.e., $\hat{v} \to v$), so the policy computed on the estimated distribution converges to optimality. We use the cumulative regret to characterize the asymptotic behavior of an online learning algorithm, which is defined as $R_T\triangleq\sum_{t=1}^T \ell(\pi^{(t)},\bm{x}^{(t)})$, where $\pi^{(t)}$ is the policy used in episode $t$. If $\mathbb{E}[R_T]$ grows only at a sublinear rate, then the averaged simple regret converges to $0$, proving the convergence of an online algorithm.

\subsection{Missing Data Problem in Data-generating Distribution Estimation} 
\label{sec:missingdata}

One might be tempted to fit the OTP into a standard episodic MDP framework \citep{azar2017minimax}, since both involve learning an optimal policy $\pi^*$ from transitions governed by an unknown kernel $q$. However, a key distinction lies in the reward feedback: in standard episodic MDPs, the reward is observed for each state and action pair, whereas in OTP, the reward for the final decision $f(\bm{x}, y)$ is typically unobserved because the $\bm{x}$ might not be revealed.

This lack of reward feedback is not a technical artifact but a fundamental characteristic of real-world testing scenarios. The goal of testing is to infer the correct decision, so the ground truth is, by design, hidden. For instance, in the quality control example shown in Figure \ref{fig:inspection}, manufacturers rarely perform all available tests on parts that fail early inspection, making it impossible to exactly know whether those parts were actually defective or misclassified. Similarly, in medical diagnosis, doctors often do not receive definitive feedback on whether a diagnosis was correct, especially when symptoms resolve on their own or treatment effects are ambiguous. In the case of gout, for example, a patient may recover in one to two weeks regardless of treatment, so even the patient might never know if the prescribed medication was appropriate.

From a modeling perspective, if there existed a reliable and inexpensive way to reveal the correctness of a decision, it would itself be treated as a test. If such a shortcut were cheap and always available, the testing problem would become trivial: one would simply invoke it every time. But in reality, these shortcuts are either costly, noisy, or unavailable, which reinforces the need to make decisions under uncertainty while balancing information gain and cost.

Due to this missingness in the feedback of $y$, we cannot use the existing algorithms for episodic MDP, and it is unclear how existing algorithms can be adapted to incorporate the missingness. Yet, as suggested by Algorithm \ref{alg:dp}, instead of learning rewards and transition probabilities separately, one can generate a policy by directly estimating $\mathcal{P}$. Intuitively, one can constantly update the estimation of $\mathcal{P}$ with freshly acquired data and improve the decision quality. However, one caveat exists in this method: the test outcome collected in operation has missingness and can introduce bias to the dataset (see an example in Table  \ref{tab:dataset}). Whenever a policy decides to terminate without doing all the tests, the observation of the random vector $\bm{x}$ is left with missing entries. Because the termination is a decision based on the observed entries, this is a typical case of Missing-At-Random \citep[MAR,][]{little2019statistical} by design, where the missingness is dependent on the observed data. Such missingness makes it challenging to estimate the distribution with incomplete data. Unlike standard episodic MDP methods that directly estimate transition probabilities and rewards for state-action pairs, \textit{decision making is now tangled with estimation since the missing pattern is created by the policy one uses}.

\subsection{Mini-Max Lower Bound} \label{sec:minmax}
In this section, we prove the mini-max lower bound of cumulative regret for the OTP. Without further assumptions, OTP is proven harder than standard episodic MDP in terms of its regret lower bound. 

At the core of it all, this hardness is a direct consequence of the conflict between exploration and exploitation in this problem. In most online learning regimes, the agent usually gets feedback on at least the reward of the policy played, so one can at least estimate the quality of the policy by playing it. For example, in a multi-armed bandit problem \citep[MAB,][]{lattimore2020bandit}, even if one chooses to exploit by playing the arm with the highest average reward, they still observe the outcome of that arm, which helps refine its reward estimate. However, exploration and exploitation in the OTP never align with each other. \textit{To exploit, the agent skips a test and makes a final decision with observed evidence. Since the test is skipped, the agent learns nothing about the missing test in such a round.} 

We first consider the simple case where there is only one test (in Theorem \ref{thm:lb_single}), and we then create examples of $d$ tests by stacking the one-dimensional examples (in Theorem \ref{thm:lb}). Suppose a manufacturer (i.e., agent) can use only one test to determine if a part is qualified or not, whose outcome is binary. Specifically, the test outcome $x_1$ is a Bernoulli variable which takes realizations in $\{0,1\}$, with a probability of $p_0$ and $1-p_0$, respectively. If $x_1=1$, then the part is qualified (i.e., $y=1$). Otherwise, if $x_1=0$, then the part is defective and should be rejected (i.e., $y=0$). For simplicity, any false decision results in a loss of $1$. To translate into our formulation, this means that $f(1,0) = f(0,1) = 0$ and $f(0,0)=f(1,1)=1$.

Because there is only one test in every episode, the agent can choose either to perform the test and make the correct decision based on the observed $x_1$ or make a decision without even testing (see Figure \ref{fig:single_lb}).
\begin{figure}[htb]
    \centering
    \includegraphics[width=0.3\linewidth]{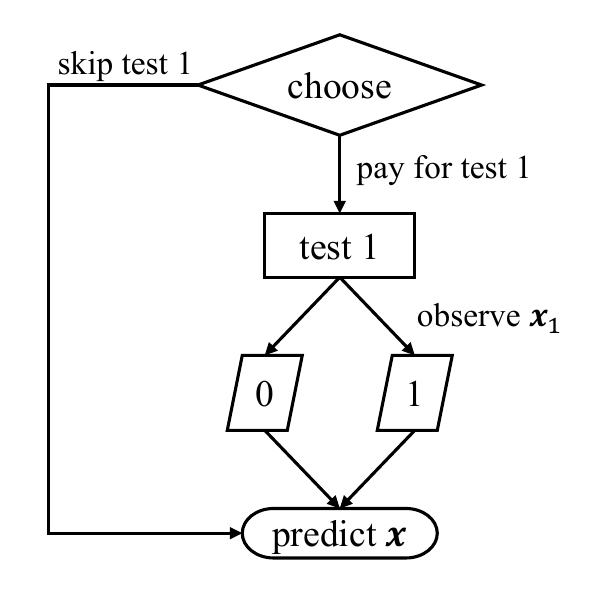}
    \caption{Lower Bound Example (single test)}
    \label{fig:single_lb}
\end{figure}

If the agent chooses to do the test, it must pay a test cost of $c_1 = \frac{3}{4}$. If the agent chooses to make a decision without a test, it might incur a loss due to a false decision.  The \textit{expected} cost (i.e., negative reward) of the three actions can be seen below:

\begin{enumerate}
    \item Action 1: skip the test and predict 0: $\mathbb{E}_{x_1}[f(x_1,0)] = 1-p_0$
    \item Action 2: skip the test and predict 1: $\mathbb{E}_{x_1}[f(x_1,1)] = p_0$
    \item Action 3: do the test (and predict whatever is observed): $c_1 = \frac{3}{4}$
\end{enumerate}

We consider the following two problem instances $\mathcal{I}_1$ and $\mathcal{I}_2$: \[\mathcal{I}_1: p_0 = \frac{1+ \epsilon}{2}, \quad \mathcal{I}_2: p_0 = \frac{1-\epsilon}{2},\]
where $\epsilon>0$ is a small number to be specified later. 

We first consider a simpler version of this problem where the task is only to make one decision. That is, after $N$ observations of $x_1$ are made (i.e., Action 3 is played $N$ times), the agent predicts the action of the lowest cost out of the three.   Clearly, since the observations of $x_1$ are i.i.d., to ease notations, we can assume without loss of generality that the first $N$ samples are observed and one decision is made at $t = N+1$.

First, we define some notations to set up the ground for the proof. Consider the measurable space $(\Psi,\mathcal{F})$ where sample space $\Psi = \{0,1\}^N$ and $\mathcal{F}$ is the power set of $\Psi$. We denote the probability of event $\psi \in \mathcal{F}$ under problem instance $\mathcal{I}_i$ to be \(P_i(\psi)  = \mathbb{P}(\psi\mid \mathcal{I}_i).\) Each problem instance $\mathcal{I}_i$ defines a $P_i$ over $\mathcal{F}$. Following this definition, we let $g:\Psi\to\{1,2\}$ be an estimator of $x_1$ after $N$ observations of $x_1$. 

First, we want to show in Lemma \ref{thm:fail_lb} that any deterministic $g$ should make mistakes with a considerable probability on at least one of the problem instances. Since the loss of any randomized policy is just a convex combination of some deterministic policies, this result generalizes to randomized policies as well. The proof of all the Lemmas and Theorems is deferred to the Appendix. The proof of Lemma \ref{thm:fail_lb} can be found in Section \ref{apx:single_lb} of the Appendix.

\begin{lemma}
\label{thm:fail_lb}
Suppose $N\leq \frac{1}{8\epsilon^2}$. For any $g$, there exists at least one problem instance $\mathcal{I}_i$ such that
\[P_i(g(\psi)=i)< \frac{3}{4}.\]
\end{lemma}

Now, we can prove a lower bound on the mix of $\mathcal{I}_1$ and $\mathcal{I}_2$. Consider a uniform distribution over the two problems, where each has a probability of $0.5$. On the one hand, according to Lemma \ref{thm:fail_lb}, any algorithm that does not have a sufficient number of tests will make a mistake on at least one of the problem instances with a constant probability. On the other hand, if we do a lot of tests, every test will also incur a constant regret. As such, we show in Theorem \ref{thm:lb_single} that the regret should scale with at least $\Omega(T^{\frac{2}{3}})$. The proof of Theorem \ref{thm:lb_single} can be found in Section \ref{apx:single_lb} of the Appendix.

\begin{theorem}
\label{thm:lb_single}
Fix time horizon $T$. For any algorithm $\pi$, there exists a problem instance such that $\mathbb{E}[R_T]$ is at least $\Omega(T^{\frac{2}{3}})$.
\end{theorem}

Now, to construct the lower bound example that gives dependency on the size of the state space, we construct an example of two tests where the \textit{second test is free} but correlated to the first test (see Figure \ref{fig:multi_lb}). For simplicity, we consider only the algorithms that start by doing test $2$. Because test $2$ is free, it does not hurt to start by doing test $2$. For any optimal policy that does not start by testing $x_2$, we can always find another policy that starts by doing so. If we can prove a lower bound for algorithms that start by testing $x_2$, it will also be a lower bound for all the algorithms. We let $x_2$ take values in $\{1,\dots, \frac{|\mathcal{P}|}{2}\}$, so after doing test 2, we will arrive in $\frac{|\mathcal{P}|}{2}$ states depending on the outcome of test 2. For example, if $x_2=1$, we will be in state $\bm{s} = (\texttt{NA},1)$.

We let the marginal probability of $\mathbb{P}(x_2 = i) =\frac{2}{|\mathcal{P}|}$ for all $i\in [\frac{|\mathcal{P}|}{2}]$. To construct hard instances, we let the conditional distribution of $x_1|x_2$ be either $\mathcal{I}_1$ or $\mathcal{I}_2$ shown previously in the single test case. Because $x_2$ has $\frac{|\mathcal{P}|}{2}$ different realizations, each problem instance in Figure \ref{fig:multi_lb} can have $\frac{|\mathcal{P}|}{2}$ conditional distributions. Hence, there are in total $2^{\frac{|\mathcal{P}|}{2}}$ problem instances constructed, and we assign a uniform distribution over them with each having a probability of $2^{-\frac{|\mathcal{P}|}{2}}$.

The idea of proving the lower bound is to treat every single state of $x_2$ as a separate problem instance we constructed above. If we can equally distribute the realizations across all the $|\mathcal{P}|$, then each subproblem should have approximately $\frac{T}{|\mathcal{P}|}$ rounds, which contributes $O\left(\frac{T^{\frac{2}{3}}}{|\mathcal{P}|^{\frac{2}{3}}}\right)$ regret. By summing up the regrets from all the states, we get the desired outcome.

\begin{figure}
    \centering
    \includegraphics[width=0.7\linewidth]{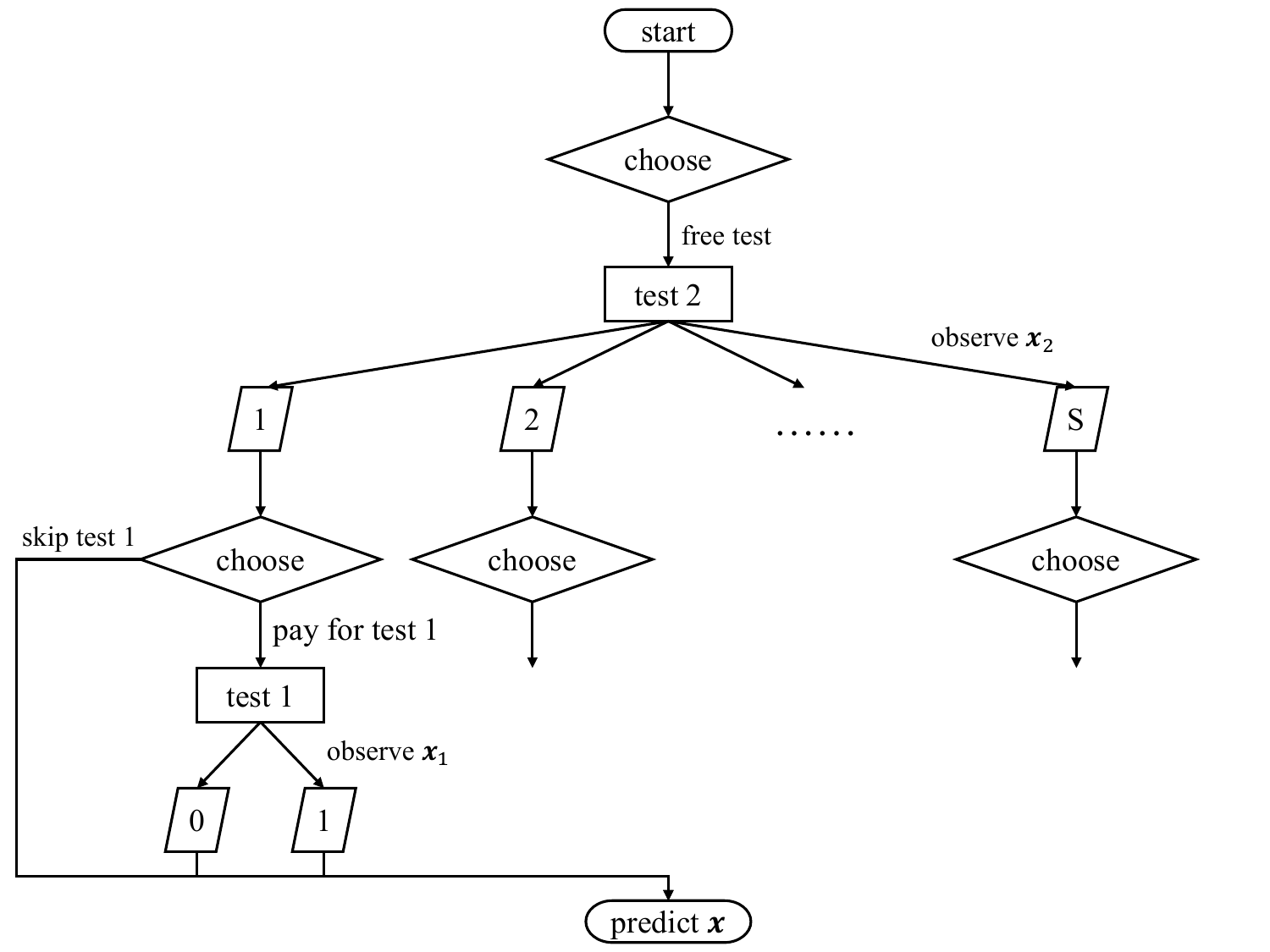}
    \caption{Lower Bound Example (multiple tests)}
    \label{fig:multi_lb}
\end{figure}

\begin{theorem}
\label{thm:lb}
Fix any time horizon $T$ and any $\kappa>0$. For any algorithm, there exists a problem instance with the support of $\mathcal{P}$ being size $\frac{|\mathcal{P}|}{2}$ such that $\mathbb{E}[R_T]$ is at least $\Omega(|\mathcal{P}|^{\frac{1}{3}}T^{\frac{2}{3}})$.
\end{theorem}

Compared to Theorem \ref{thm:lb_single}, Theorem \ref{thm:lb} shows not only the dependence on the $T$, but also the dependence on the support size $|\mathcal{P}|$. The dependence on $|\mathcal{P}|$ turns out to be tight as it is matched by an upper bound in Theorem \ref{thm:etcG} in the next section.

\subsection{Methodology: Explore-Then-Commit} \label{sec:etcproof}

As mentioned earlier, one fundamental reason for having this elevated lower bound is that there is a direct conflict between exploration and exploitation. Therefore, it is natural to have an algorithm that separates the two factors, namely the Explore-Then-Commit (ETC) algorithm. To avoid the problem of missing data, ETC does every possible test for a series of subjects to learn the data-generating distribution $\mathcal{P}$, which is called an exploration phase. After the exploration phase, we can estimate the distribution $\mathcal{P}$ unbiasedly and commit to a policy by solving Algorithm \ref{alg:dp}.

\begin{algorithm}
\caption{Explore-Then-Commit (ETC)}
\label{alg:etcD}
\begin{algorithmic}[1]
\State \texttt{\# Explore}
\State \textbf{Input}: Total episodes $T$
\State Set $N=\lfloor |\mathcal{P}|^{\frac{1}{3}}T^{\frac{2}{3}} \rfloor$
\For{$t=1,\dots,N$}
\State Do all the tests, i.e., observe the entire $\bm{x}^{(t)}$
\State Make decision $y_t = \argmax_{y}f(\bm{x}^{(t)},y)$
\EndFor
\State \texttt{\# Commit}

\State Compute $\hat{\mathcal{P}}$ by setting the probability mass function of $\bm{x}$ as $\mathbb{P}(\bm{x}) = \frac{\sum_{t=1}^{N}\mathds{1}\{\bm{x}^{(t)} = \bm{x}\}}{N}$

\State Compute policy $\hat{\pi}$ using Algorithm \ref{alg:dp} and $\hat{\mathcal{P}}$.

\For{$t=N+1,\dots,T$}
\State Play policy $\hat{\pi}$.
\EndFor
\end{algorithmic}
\end{algorithm}

This algorithm turns out to match the regret lower bound on episode $T$ and support size $\mathcal{P}$. We prove this argument in Theorem \ref{thm:etcD} below.

\begin{theorem}
\label{thm:etcD}
    Given time horizon $T$, number of tests $d$, and bounded costs $\{c_i\}_{i\in [d]}$ and reward function $f$, for any fixed discrete distribution $\mathcal{P}$ with finite support, with a probability of at least $1-\delta$, the cumulative regret of running Algorithm \ref{alg:etcD} is at most $\tilde{O}(d|\mathcal{P}|^{\frac{1}{3}}T^{\frac{2}{3}})$.
\end{theorem}

Here we give a proof sketch of Theorem \ref{thm:etcD}, and the complete proof is deferred to Section \ref{apx:etcD} in the Appendix. While the dependence on $\tilde{O}(T^{\frac{2}{3}})$ is expected for ETC algorithms, the main challenge lies in the dependence of $|\mathcal{P}|$ and $d$. To the best of our knowledge, there exist no results for ETC algorithms in even episodic MDP literature since it is usually suboptimal compared to the lower bound $\Omega(\sqrt{T})$. Instead of considering the state and action dependence, we take a different route here by examining the dependence on the support size $|\mathcal{P}|$, which ends up giving a sharper bound. Such a characterization is also tighter and provides more insights when the data-generating distribution $\mathcal{P}$ is sparse. 

To obtain the desired results, one significant step is to bound the error of the reward functions and the transition probabilities estimated using $\hat{\mathcal{P}}$, which boils down to bounding the $L_1$ deviations of the estimated posterior distributions $\hat{P}^{\bm{s}}$. By a careful analysis of the state space, we first show that although the total number of posterior distributions (i.e., $\hat{P}^{\bm{s}}$) to be estimated can be large, the number of non-trivial $\hat{P}^{\bm{s}}$ is at most $|\mathcal{P}|^2$. With this observation, we apply a union bound on the standard $L_1$ deviation bound for empirical distributions \citep{weissman2003inequalities} and get an error bound of $\tilde{O}\left(\sqrt{\frac{k(\bm{s})}{N(\bm{s})}}\right)$, where $k(\bm{s})$ is the size of the support of posterior $\hat{P}^{\bm{s}}$ and $N(\bm{s})$ is the number of samples $\bm{x}$ that satisfies $\bm{x}\overset{\Re}{=}\bm{s}$. 

To aid the proof, we introduce the following new notations.  We use $\hat{P^{\bm{s}}}$, $\hat{\mathbb{E}}[r(\bm{s},y)]$, $\hat{Q}(\bm{s},i)$, and $\hat{v}(\bm{s})$ to denote the estimations of $P^{\bm{s}}$, $\mathbb{E}[r(\bm{s},y)]$, $Q(\bm{s},i)$, and $v(\bm{s})$. Moreover, we use $\Delta \mathbb{P}$, $\Delta \mathbb{E}[r(\bm{s},y)]$, $\Delta Q(\bm{s},i)$, and $\Delta v(\bm{s})$ to denote the absolute deviation of the estimations. For example, we have $\Delta v(\bm{s}) \triangleq |v(\bm{s}) - \hat{v}(\bm{s})|$.

The regret of the algorithm is dictated by the accuracy of value functions $\hat{v}$ (or equivalently $\hat{Q}$). If tests are terminated in state $\bm{s}$ and a final decision is made, the error of the corresponding action is determined by the estimation error in Equation \ref{eq:rdecision}, which is of order $\tilde{O}\left(\sqrt{\frac{k(\bm{s})}{N(\bm{s})}}\right)$. If a test is chosen, its error is determined by the estimation error of $\hat{Q}$ in Equation \ref{eq:qvalue}. We decompose the Equation \ref{eq:qvalue} into three terms: the first term depends on the estimation error of the next state $v(\bm{s}\oplus(i,x_i))$; the second term depends on the estimation error of posterior $\hat{P}^{\bm{s}}$, and the last term is the product of the two errors, which is considered a higher order term. Since the error in $\hat{v}$ is upper bounded by the error in $\hat{Q}$, we can see that the estimation error propagates at most $d$ times through $Q$ in dynamic programming. However, instead of introducing another $d$ dependence through the propagation, we show that the error in $\hat{v}(\bm{s})$ and $\hat{Q}(\bm{s},\cdot)$ can be bounded also by $\tilde{O}\left(\sqrt{\frac{k(\bm{s})}{N(\bm{s})}}\right)$. Obviously, $k(\bm{s})\leq |\mathcal{P}|$, so the expected regret of one exploitation step can hence be controlled roughly by $O\left(d\sqrt{\frac{|\mathcal{P}|}{N}}\right)$, where $N$ is the number of exploration steps. Finally, we choose $N=|\mathcal{P}|^{\frac{1}{3}}T^{\frac{2}{3}}$ to balance between the regret in exploration and exploitation, and obtain a final regret bound of $\tilde{O}(d|\mathcal{P}|^{\frac{1}{3}}T^{\frac{2}{3}})$.

It should be highlighted that this upper bound matches the lower bound's dependence on $|\mathcal{P}|$ and $T$ shown in Theorem \ref{thm:lb} up to logarithmic factors. This provides a theoretical justification for using a ``data collecting - model training'' approach like \cite{yu2023deep}. Here, ``data collecting'' refers to the exploration phase in the algorithm, and ``modeling training'' refers to the exploitation phase. That said, one still needs to be cautious about the missingness in the collected dataset, especially when some tests are early terminated. Such missingness appears not only in the deployment of trained policies, as we have explained in Section  \ref{sec:missingdata}, but it also naturally appears when tests are early terminated by humans. This is a typical example in a medical diagnosis dataset, where doctors obviously do not recommend all the tests for patients.

Note that Theorem \ref{thm:etcD} does not automatically generalize to cases where $\mathcal{P}$ are continuous since the support of $\mathcal{P}$ becomes uncountable. In fact, $L_1$ deviations of empirical distributions do not converge in continuous settings, making it impossible to obtain the same convergence rate as in the discrete setting without further assumptions. As such, we show one example of a parametric case where the $\mathcal{P}$ is a multivariate Gaussian distribution, which also provides a comparison with the special case shown in Theorem \ref{thm:itr}. 

We assume without loss of generality that the covariance matrix has a finite condition number of $\sigma$. Since $\Sigma$ is a covariance matrix, its eigenvalues are at least non-negative, so this assumption effectively excludes the case where the covariance matrix is singular (i.e., $\det(\Sigma)=0$). For a multivariate Gaussian distribution, if the covariance matrix is singular, it means that there is at least one eigenvalue of $\Sigma$ equal to zero, which corresponds to some redundant dimensions. This assumption is without the loss of generality because the redundant dimensions can be identified with high probability by taking $d$ full samples, which only contributes to an additive constant regret and will not affect our analysis on cumulative regrets.

\begin{algorithm}
\caption{Explore-Then-Commit (ETC), Gaussian}
\label{alg:etcG}
\begin{algorithmic}[1]
\State \textbf{Input}: Total episodes $T$
\State \texttt{\# Explore}

\State Set $N=\lfloor \sigma^2 T^{\frac{2}{3}} \rfloor$
\For{$t=1,\dots,N$}
\State Do all the tests, i.e., observe the entire $\bm{x}^{(t)}$
\State Make decision $y_t = \argmax_{y}f(\bm{x}^{(t)},y)$
\EndFor
\State \texttt{\# Commit}

\State Estimate $\hat{\mathcal{P}}$ using the plug estimator where $\hat{\bm{\mu}} = \frac{1}{N}\sum_{n=1}^N \bm{x}^{(n)}$ and $\hat{\Sigma} = \frac{1}{N}\sum_{n=1}^N \bm{x}^{(n)}{\bm{x}^{(n)}}^\intercal$.

\State Compute policy $\hat{\pi}$ using Algorithm \ref{alg:dp}.

\For{$t=N+1,\dots,T$}
\State Play policy $\hat{\pi}$.
\EndFor
\end{algorithmic}
\end{algorithm}

\begin{theorem}
\label{thm:etcG}
    Given time horizon $T$, number of tests $d$, and bounded costs $\{c_i\}_{i\in [d]}$ and reward function $f$, for any fixed Gaussian distribution $\mathcal{P}$ of condition number $\sigma$, with a probability of at least $1-\delta$, the cumulative regret of running Algorithm \ref{alg:etcG} is at most $\tilde{O}(d\sigma^2T^{\frac{2}{3}})$.
\end{theorem}

The main steps in the proof of Theorem \ref{thm:etcG} are similar to the proof of Theorem \ref{thm:etcD}. The complete proof is deferred to Section \ref{apx:etcG} in the Appendix. The key difference lies in how one can establish error bounds on the $P^{\bm{s}}$. Since we no longer build empirical distributions for $P^{\bm{s}}$, its estimation error is dictated by the error in the estimated mean and covariance matrices. Gaussian distribution is a special case where posterior distributions $P^{\bm{s}}$ have closed forms, but the $L_1$ deviations of the posterior distributions $P^{\bm{s}}$ depend not only on the estimation error of the estimated parameters, but also on the norm of $\bm{s}$, which can be arbitrary large in the worst case. Nonetheless, because we consider only the expected regret, the probability of such cases is small in the sense that the effect of $\bm{s}$ can be controlled by the conditioning number of the covariance matrix. 

Arguably, one limitation of ETC shown in Algorithm \ref{alg:etcD} and \ref{alg:etcG} is that it takes the time horizon $T$ as input. In some practical settings, the number of episodes $T$ might be unclear. One straightforward adaptation to make ETC an any-time algorithm is to use the doubling trick \citep{CesaBianchiLugosi2006}, which divides the time horizon into exponentially growing batches of $2^0,2^1,\dots$ and applies the standard ETC algorithm for each batch. The resulting algorithm will enjoy the same theoretical guarantee without taking $T$ as an input.

\subsection{Numerical Experiments}

We verify our theoretical results with a numerical simulation on a randomly generated instance. The instance has $d=10$ tests where each test gives a binary outcome. The probability of each certain realization is randomly sampled from a Pareto distribution with a shape parameter of $\log_4(5)$ and normalized to keep their summation equal to one. The reward is also set as binary in this simulation, where $f(\bm{x},\bm{y}) = \mathds{1}(\bm{x}=\bm{y})$. In other words, the algorithm is given a reward of $1$ if it can correctly impute $\bm{x}$ at the end of tests, and it gets a $0$ reward otherwise. The horizon is set to be $T=2^{20}$. Figure \ref{fig:general_simulation} reports the mean (in solid lines) and the standard deviations (in shaded areas) of the cumulative regret for our proposed algorithms across five replications. 

The numerical simulation result verifies our theory. The blue line represents the performance of Algorithm \ref{alg:etcD} when $T$ is given, whereas the orange line is using Algorithm \ref{alg:etcD} with the doubling trick, which does not require a predefined $T$. Both algorithms show a clear sign of sublinear growth in their cumulative regret. 

The ETC algorithm, when given episodes $T$, has better empirical performance. Although the ETC algorithm has exactly the same number of explorations in this numerical simulation, ETC with known episodes spends all its commit steps after explorations, which gives a smaller error in the estimation of the distribution $\mathcal{P}$ and thus a smaller regret. 

\begin{figure}
    \centering
    \includegraphics[width=0.7\linewidth]{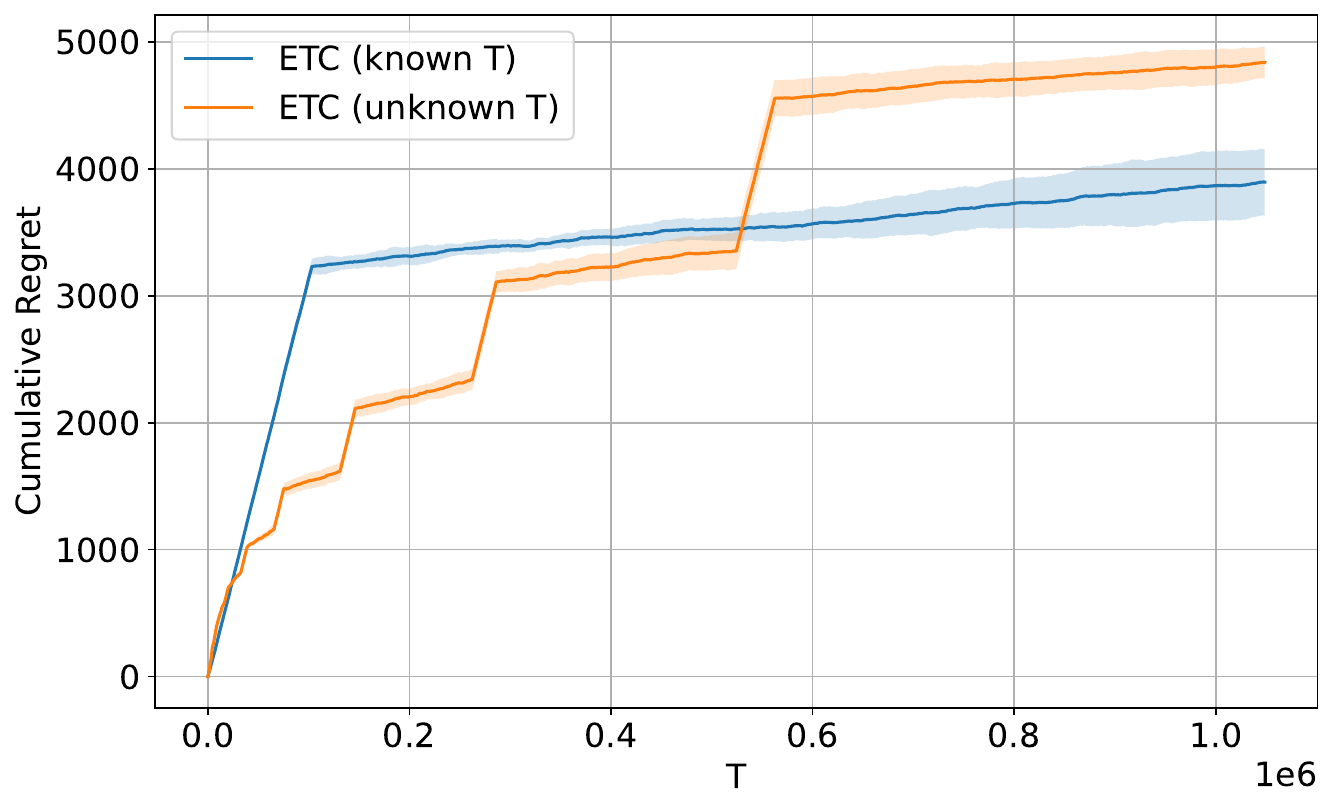}
    \caption{Cumulative regret of Algorithm \ref{alg:etcD} in a discrete setting}
    \label{fig:general_simulation}
\end{figure}

\section{Online Cost-sensitive Maximum Entropy Sampling Problem}
\label{sec:onlineMESP}
It is natural to ask whether the difficulty observed in the general OTP is inherent across all problem instances. Can the elevated regret be mitigated under structural assumptions or simplified settings? To explore these questions, we examine a specialized yet practically meaningful variant: the Online Cost-sensitive Maximum Entropy Sampling Problem (OCMESP). OCMESP retains many key features of OTP, including sequential test selection, cost-sensitive information acquisition, and partial observability, but introduces a crucial simplification: the reward function $f$ depends only on observed outcomes. In doing so, OCMESP serves as a theoretical ablation study of the general OTP, isolating the impact of missing rewards, enabling the design of algorithms with improved regret bounds (specifically, $\tilde{O}(d^3\sigma\sqrt{T})$), and offering a useful lens through which to understand which components of OTP contribute most to its hardness.

\subsection{Problem Formulation}

OCMESP is a generalized online version of the classical maximum entropy sampling problem \citep[MESP,][]{fampa2022maximum}. The classical maximum entropy sampling problem assumes the data $\bm{x}$ are generated by a \textit{given} multivariate normal distribution $\mathcal{N}(0,\Sigma)$ (the means are set to zeros without loss of generality). The goal of MESP is to find the subset $S\subseteq [d]$ of size $n_s$ such that the subset has the largest information entropy. Mathematically, this can be written as 
\begin{equation}
\label{eq:MESP}
    \max_{S\subseteq [d]} ~ \mathcal{H}(S), \quad \text{s.t.} ~ |S|=n_s,
\end{equation}
where the $\mathcal{H}(S)$ is the information entropy \citep{cover-thomas2006}. When $\mathcal{P}$ is Gaussian, it has a closed form solution of $\mathcal{H}(S)=\frac{|S|}{2}\log(2\pi e) + \frac{1}{2}\log\det (\Sigma[S,S])$, where $\Sigma[S,S]$ denotes the covariance matrix of the selected tests and can be generated by selecting the corresponding rows and columns from $\Sigma$. 

In this paper, rather than putting a hard constraint on the number of tests one can do (as shown in Equation \ref{eq:MESP}), we consider a more general version that weighs information gain and test costs as follows:
\begin{equation}
\label{eq:cMESP}
    \max_{S\subseteq [d]} ~ \lambda \cdot \mathcal{H}(S) - \sum_{i \in S} c_i,
\end{equation}
where $\lambda$ is the reward per unit of information. Same as the general OTP, we still aim to choose a sequence of tests $S$, and we can only observe the outcomes of the chosen tests (denoted as $\bm{x}[S]$). 

Similar to OTP, this problem can also be formulated as an MDP as shown in Figure \ref{fig:mdp_ocmesp}. The vertical and horizontal sequentiality holds the same in OCMESP, while the \textit{only} difference is that one does not have to make a decision $y$ at termination. The rewards of the termination step are now defined by $\lambda \cdot \mathcal{H}(S)$. To be aligned with the rewards in OTP, we can rewrite entropy $\mathcal{H}(S)$ as $\lambda \mathbb{E}[-\log p(\bm{x}[S])]$, so this can be treated as a special case of OTP when $f(\bm{x}, S) = \lambda \mathbb{E}[-\log p
(\bm{x}[S])]$. Note that the reward can still not be directly observed because the distribution $\mathcal{P}$ is unknown (and so is its entropy). Thus, it should be estimated using the observed $\bm{x}[S]$.

\begin{figure}
    \centering
    \includegraphics[width=0.8\linewidth]{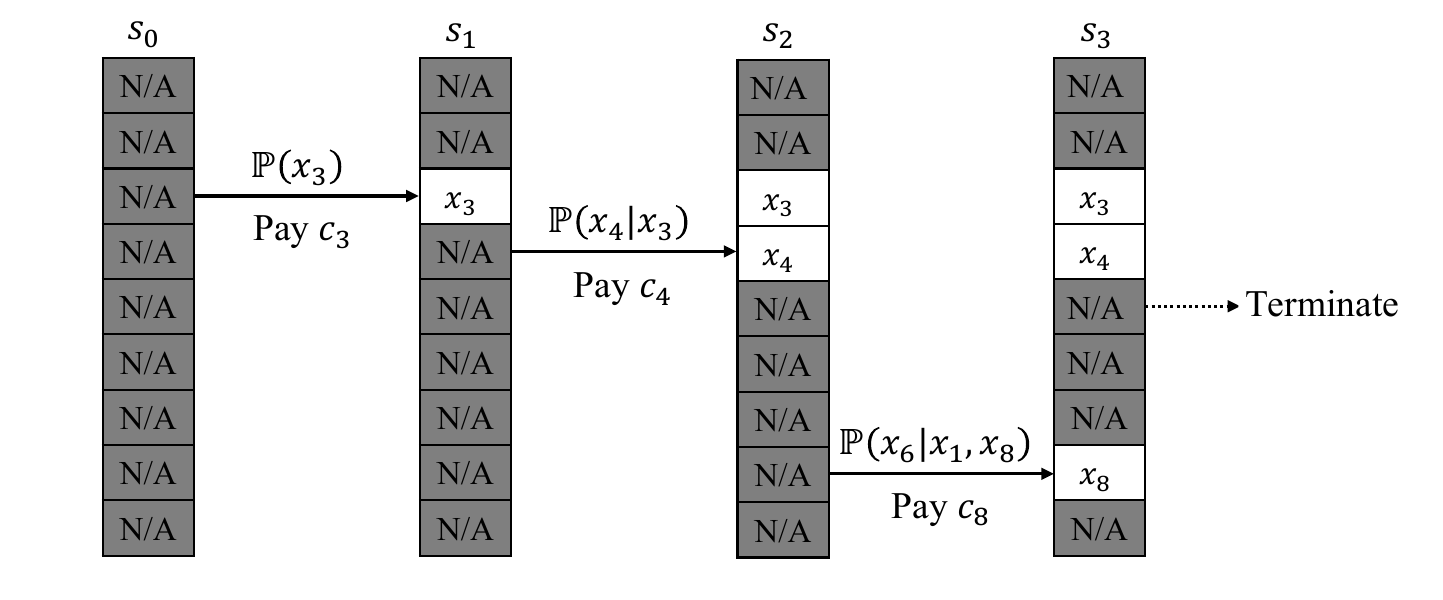}
    \caption{MDP Formulation of OCMESP}
    \label{fig:mdp_ocmesp}
\end{figure}

\subsection{Commonalities and Differences with General Online Testing Problems}
\label{sec:connection}

As mentioned, OCMESP can be treated as a special case of the general OTP if we specify the reward function. OCMESP thus shares some important properties with OTP. First, both problems require balancing information acquisition and testing costs.  Second, both problems select only a subset of tests to evaluate, which leaves missing entries in the dataset. Consequently, the dataset collected in the OCMESP problem will still be like Table  \ref{tab:dataset}. Third, both problems have missing rewards for some actions. For OTP, rewards are missing because they depend on the missing entries, while all the rewards are missing in OCMESP since the distribution $\mathcal{P}$ is unknown, and entropy is not directly observable.

Despite the commonalities, the fundamental difference between the two problems lies in whether or not the missing rewards can be accurately estimated by repeated plays. In OCMESP, the reward function depends only on the test outcomes that are actually observed (i.e., $\bm{x}[S]$) but not the missing ones, i.e., $f(\bm{x},S)=f(\bm{x}[S],S)$. As a result, even though one may still not observe the reward $f(\bm{x},S)$ directly, this quantity can be estimated by repeatedly testing $S$ on multiple subjects since one can estimate $\Sigma[S,S]$ with samples of $\bm{x}[S]$. In other words, repeating the same testing strategy allows one to estimate its average reward over time.

In contrast, OTP is fundamentally harder. The reward depends on the full vector $\bm{x}$, including any tests one chooses to skip. If one makes a decision $y$ without performing all tests, the reward $f(\bm{x},y)$ is missing because $\bm{s}$ is not fully observed. Repeating the same set of tests $S$ doesn't help get a more accurate estimation of $r(\bm{s},y)$ because there is no information to determine whether $y$ was correct or not. In fact, a correct estimation of $r(\bm{s},y)$ does not only rely on an accurate estimation of the distribution of $\bm{x}[S]$, but also on the distribution of the entire $\mathcal{P}$. This inability to estimate rewards through repeated play of a strategy is precisely why the minimax regret in OTP scales faster with time (e.g., $\Omega(T^{\frac{2}{3}})$), compared to more favorable cases like OCMESP.

\subsection{Iterative Elimination Algorithm}
\label{sec:itr}

The structure of OCMESP opens the door to more efficient learning algorithms. Since rewards can be consistently estimated from repeated actions, we can sidestep the core challenge of missingness that complicates the general OTP. Building on this observation, we now propose an iterative elimination algorithm that incrementally refines the set of candidate test subsets using plug-in estimators of entropy and their confidence bounds, ultimately achieving a regret of $\tilde{O}(d^3\sigma\sqrt{T})$.

Since the reward function depends only on the observed entries, the reward of a particular $S$ can be well-approximated by a plug-in estimator of its differential entropy. As such, one can roughly treat this problem as an MAB where every chosen set of tests $S$ can be treated as an arm, although the feedback of the arm is not directly observable. One problem with treating it as a naive MAB problem is that the number of arms in this problem is exponentially larger in $d$, introducing an exponential dependence on $d$ in the cumulative regret. Fortunately, the rewards of most arms are correlated with each other since they are dictated by only $d^2$ parameters. We note the following three properties that will help us design an efficient algorithm.

The first important property is that there exists a non-adaptive policy that is as good as the best adaptive policy. Although one can adaptively choose the next test  (or choose to terminate) based on the observations of the previous test outcomes, the optimal policy requires only choosing the optimal subset $S$, which involves no adaptivity. As such, if we use a non-adaptive policy, the missing data pattern will be Missing-Completely-At-Random \citep[MCAR,][]{little2019statistical} rather than MAR because the missingness no longer depends on the realization $\bm{x}$. Therefore, we can get unbiased estimations of the distribution $\mathcal{P}$ even with missingness in the dataset. 

The second important property is that choosing a set $S$ effectively gives information on all of its subsets. We use the set $\mathcal{C}$ to denote the feasible choice of $S$. For any set $\mathcal{C}$, we call a $S\in \mathcal{C}$ a \textit{parent} of set $\mathcal{C}$ if there exist no $S' \in \mathcal{C}$, such that $S\subset S'$. One can see that for any set $\mathcal{C}$, the number of \textit{parents} is no larger than  $\max_{n_s}|\{S\in \mathcal{C}:|S|=n_{s}\}|$, which is the maximum number of $S$ that should be considered if there were a cardinality constraint. If we use an iterative elimination algorithm (see Algorithm \ref{alg:itr}), it suffices to choose one of its parents to maximize exploration. Therefore, the problem in Equation \ref{eq:cMESP} is no harder than the Problem in Equation \ref{eq:MESP} (in terms of learnability), although it seems to have more feasible solutions.

The third important property is that the $2^d$ possible choices of $S$ are, in fact, parameterized by only $d^2$ parameters. If we design balanced experiments that allocate $T/d^2$ samples to learn each covariance pair (and there are $d^2$ of them), one should expect the estimation error of the covariance plug-in estimator to decay at a rate of $O\left(\sqrt{\frac{d^2}{T}}\right)$.

\begin{algorithm}[htb]
\caption{Iterative Elimination for Online Maximum Entropy Sampling}
\label{alg:itr}
\begin{algorithmic}[1]
\State \textbf{Input:} $\sigma, d, \delta$
\State Initialize $\mathcal{C}$ to be power set of $[d]$.
\State Set $t = 1$

\While{$|\mathcal{C}| > 1$}
    
    \State Update $\mathcal{Q}$ using Equation \ref{eq:remain_pairs}
    \For{$(i,j)\in \mathcal{Q}$}
    \State Update $\mathcal{T}_{ij} = \{t: x^{(t)}_i \ne \texttt{NA} \text{ and } x^{(t)}_j \ne \texttt{NA}\}$ 
    \State Update $\hat{\Sigma}_{ij}^{(t)} = \hat{\Sigma}_{ji}^{(t)} = \frac{1}{\left|\mathcal{T}_{ij}\right|}\sum_{t\in \mathcal{T}_{ij}} x^{(t)}_ix^{(t)}_j$
    \EndFor
    
    \State Update pair $(i^{(t)},j^{(t)})= \argmin_{(i,j)\in \mathcal{Q}} |\mathcal{T}_{ij}|$
    \State Test $S^{(t)} =\argmax_{S: \{i,j\} \subseteq S} |S|$ and observe $\bm{x}^{(t)}[S]$
    \State Compute the confidence interval width \(U^{(t)}\) using Equation \ref{eq:ci}
    \If{$U^{(t)}\leq 1$}
    \State Compute $\hat{\mathcal{H}}^{(t)}(S)$ using Equation \ref{eq:estobj}
    \State Eliminate $S$ for which $\hat{\mathcal{H}}^{(t)}(S) + \lambda \cdot U^{(t)} \leq \max_S \hat{\mathcal{H}}^{(t)}(S)$
    \EndIf
    
    \State Set $t = t + 1$
\EndWhile
\end{algorithmic}
\end{algorithm}

With these properties, we introduce our algorithm for the online learning of problem Equation \ref{eq:cMESP} in Algorithm \ref{alg:itr}. To run the algorithm, we need to input three parameters: (1) $\sigma$: the conditioning number or an upper bound of the condition number of the covariance matrix $\Sigma$, (2) $d$: the dimension of the distribution $\mathcal{P}$, and (3) $\delta$: the high probability constant in $(0,1)$. 

The algorithm starts by initializing an active set $\mathcal{C}$, which contains all the candidate subsets. The goal of iterative elimination is to delete elements from $\mathcal{C}$ such that any remaining subset $S\in \mathcal{C}$ has small regret (to be specified later). At every iteration of the algorithm, we use $\mathcal{Q}$ to keep track of the remaining pairs that appear in at least one of the subsets in $\mathcal{C}$. Mathematically, this means 
\begin{equation}
    \label{eq:remain_pairs}
    \mathcal{Q} = \{(i,j):\exists S\in \mathcal{C} \text{ such that }  \{i,j\} \subseteq S \}.
\end{equation}
It is obvious that only the covariance of the pairs that remain in $\mathcal{Q}$ has to be learned. The covariance for every remaining pair in $\mathcal{Q}$ is estimated using the empirical estimator. We use $\mathcal{T}_{ij} = \{t: x^{(t)}_i \ne \texttt{NA} \text{ and } x^{(t)}_j \ne \texttt{NA}\}$ to denote the observation indices whose observation covers pair $(i,j)$. Then, for a pair $(i,j)$, the estimated covariance is given by $\hat{\Sigma}_{ij}^{(t)} = \hat{\Sigma}_{ji}^{(t)} = \frac{1}{\left|\mathcal{T}_{ij}\right|}\sum_{t\in \mathcal{T}_{ij}} x^{(t)}_ix^{(t)}_j$. One potential concern regarding the estimated covariance matrix is that it might not be positive definite because each entry in the covariance matrix is estimated independently. However, we can show that this situation happens only when observations are limited. In particular, if the confidence interval width is less than $1$, then the resulting covariance matrix must be positive definite.

The algorithm chooses the maximum subset that covers the least-covered pair in $\mathcal{Q}$ (ties are broken arbitrarily). This choice is designed to distribute the observations to all the pairs in a balanced manner. As a result, at time $t$, each pair should be at least covered by $\lfloor\frac{t}{d(d-1)}\rfloor$ observations. However, in practice, one can expect to have more observations than this worst-case, given that the maximum subset usually covers more than one pair. 

After testing the chosen subset, the algorithm then eliminates candidate subsets in $\mathcal{C}$ that are suboptimal with high confidence. To do this, we construct $U^{(t)}$ to uniformly bound the estimation error in $|\hat{\mathcal{H}}^{(t)}(S)-\mathcal{H}^{(t)}(S)|$ due to finite samples. It guarantees that, with high probability of at least $1-\delta$ (of the user's choice), the estimated $\hat{\mathcal{H}}^{(t)}(S)$ for any candidate subset $S$ is within $\lambda U^{(t)}$ of its true value. This justifies our elimination rule. In our algorithm, confidence interval width $U^{(t)}$ is chosen to be
\begin{equation}
\label{eq:ci}
    U^{(t)} = 8\sigma\cdot \max\left\{d^3\sqrt{\frac{\ln(\pi^2 d^2 t^2/\delta)}{ct}}, \frac{d^4(\ln(\pi^2 d^2 t^2/\delta))}{ct}\right\},
\end{equation}
where $c$ is the universal constant in Bernstein’s inequality \citep{vershynin2018high}.
A rigorous derivation of this confidence width requires arguments on matrix algebra and concentration inequalities, which are deferred to Section \ref{apx:itr} in the Appendix.  We can show that if $U^{(t)}<1$, then with a probability of at least $1-\delta$,  the estimation error $|\hat{\mathcal{H}}^{(t)}(S)-\mathcal{H}^{(t)}(S)|$ is at most $\lambda\cdot U^{(t)}$. 

If the width is no less than $1$, we skip the elimination phase and step into the next iteration. Otherwise, we are safeguarded by the confidence interval so that any eliminated subset $S$ is suboptimal with high probability. 

To eliminate a suboptimal subset, we first estimate the rewards for every subset by plugging in the estimated $\hat{\Sigma}$:
\begin{equation}
\label{eq:estobj}
    \hat{\mathcal{H}}^{(t)}(S) = \lambda \left( \frac{|S|}{2}\log(2\pi e) + \frac{1}{2}\log\det \left(\hat{\Sigma}^{(t)}[S,S]\right) \right)- \sum_{i\in S} c_i.
\end{equation}
Then, we eliminate any subset whose estimated reward is $2U^{(t)}$ less than the maximum estimated reward over all remaining subsets (i.e., $\hat{\mathcal{H}}^{(t)}(S) + 2U^{(t)} \leq \max_{S^*} \hat{\mathcal{H}}^{(t)}(S^*)$). 

It is important to note that the use of Iterative Elimination in OCMESP is made possible precisely because the reward function depends only on observed entries. This structural property allows us to repeatedly evaluate the same test subsets and reliably estimate their rewards. In Algorithm \ref{alg:itr}, we can reduce the uncertainty of the rewards in the candidate set $\mathcal{C}$ by testing only the subsets $S$ that are in the set $C$. In contrast, such an algorithm cannot be used in the general OTP setting. We can see this in the lower bound example constructed in Theorem \ref{thm:lb_single}. Since the rewards of Actions $\#1$ and $\#2$ cannot be accurately estimated without playing Action $\#3$, even if one can determine Action $\#3$ is strictly suboptimal and should eliminated, it still has to be played in order to distinguish between Actions $\#1$ and $\#2$, rendering Successive Elimination inapplicable.

\subsection{Theory and Proof Sketch}

Having introduced the Iterative Elimination algorithm, we now turn to its theoretical analysis. The key question is whether this intuitive approach can be rigorously shown to achieve sublinear regret. In Theorem \ref{thm:itr}, we establish that Algorithm \ref{alg:itr} indeed achieves a regret bound of $\tilde{O}(d^3\sigma\sqrt{T})$, confirming that OCMESP is significantly more tractable than the general OTP and highlighting the power of its structural simplification.

\begin{theorem}
\label{thm:itr}
    With a probability of at least $1-\delta$, for sufficiently large $T$, the cumulative regret of Algorithm \ref{alg:itr} is at most
    \(\Tilde{O}(d^3\sigma\sqrt{T})\).
\end{theorem}

The proof of our main theorem is sketched as follows. First, by Bernstein's inequality, we know that the empirical estimator of covariance converges to the true covariance at a rate of $\Tilde{O}(\frac{1}{\sqrt{n}})$, where $n$ is the number of samples. With some matrix algebra, we then show that the error of the plug-in estimator for the entropy is locally linear with respect to the estimation error of the covariances. Hence, the limiting behavior of reward estimation error will share the rate of $\Tilde{O}(\frac{1}{\sqrt{n}})$. Now, because we balance the observations to evenly cover every pair of tests, every pair should have at least $\lfloor\frac{T}{d^2}\rfloor$ observations up to episode $T$. Note that this is only for the worst-case analysis, and this number should be much larger in practical settings. That said, for any subset that is not eliminated, we can easily see that the regret, if there is any, should be less than the estimation error of the entropy. Finally, the cumulative regret bound can be achieved by simply summing up the simple regrets, which ends the proof of Theorem \ref{thm:itr}. 

\subsection{Numerical Experiments}

Our theoretical result is verified by a numerical simulation on a randomly generated instance. The instance has $d=15$ tests where the joint distribution of the tests is $\mathcal{N}(0,\Sigma)$. The covariance $\Sigma$ is randomly generated by $\Sigma = LL^\intercal + I$, where the elements in $L$ are uniformly generated from $0$ to $1$. We let $\lambda=1$ and showed the performance of the algorithm up to time  $T=2^{15}$. Figure \ref{fig:cMESP} reports the mean (in solid line) and the standard deviations (in shaded area) of the cumulative regret across five replications for running Algorithm \ref{alg:itr}. We can clearly observe a sublinear convergence of cumulative regrets, as predicted by Theorem \ref{thm:itr}.

\begin{figure}[htb]
    \centering
    \includegraphics[width=0.7\linewidth]{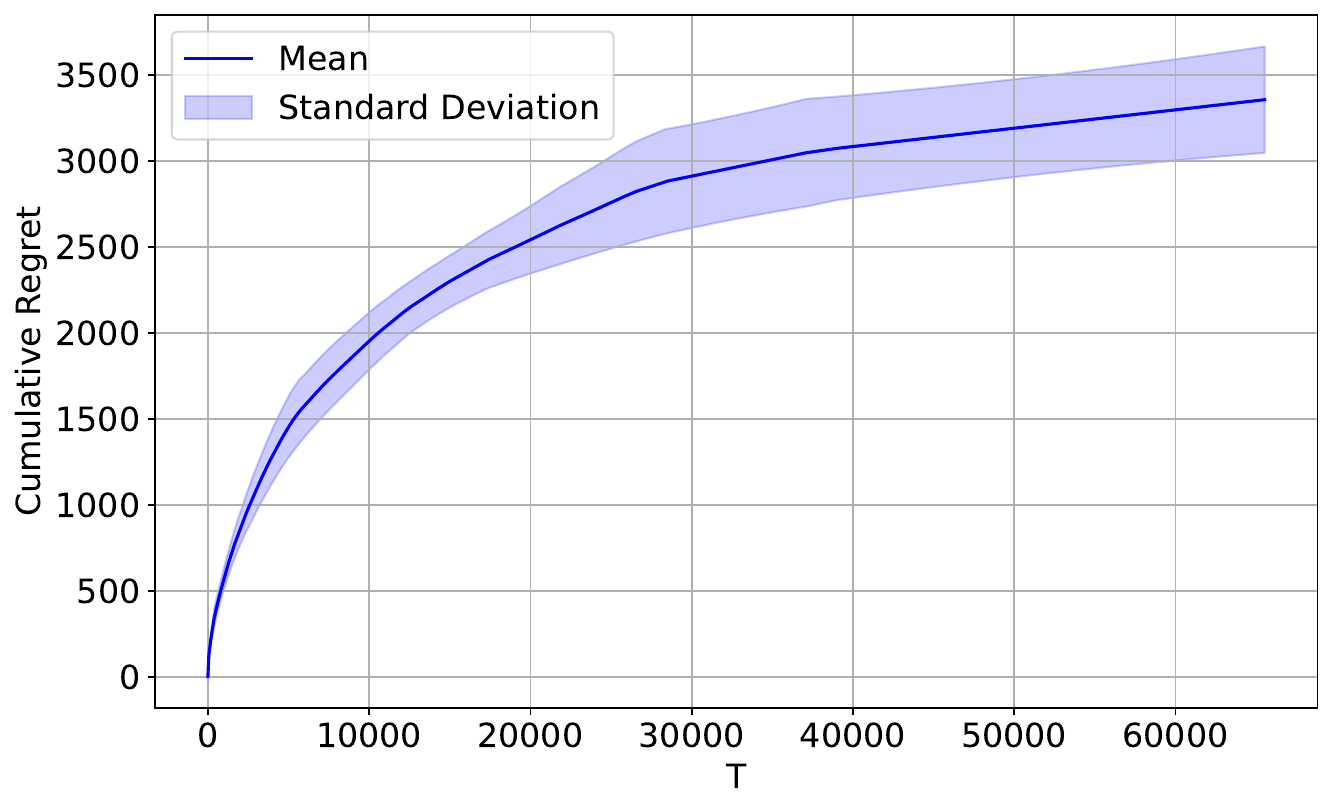}
    \caption{Cumulative Regret of the Iterative Elimination Algorithm}
    \label{fig:cMESP}
\end{figure}

\section{Conclusion}

This work advances the theoretical understanding of the Online Testing Problem.  We establish that the OTP is fundamentally harder than standard episodic MDPs: any algorithm must incur minimax cumulative regret at least on the order of $\Omega(T^\frac{2}{3})$ due to the fact that the rewards depend on the missingness.

To meet this tougher lower bound, we introduced an Explore-Then-Commit strategy.  By fully testing subjects and leaving no missingness, we learn the data-generating distribution $\mathcal{P}$ and then commit to the dynamic-programming policy derived from that estimate. ETC achieves regret $\tilde{\mathcal{O}}(d |\mathcal{P}|^{\frac{2}{3}} T^{\frac{2}{3}})$ for discrete distributions and $\tilde{\mathcal{O}}(d \sigma^2 T^{\frac{2}{3}})$ for Gaussian outcomes, matching the lower bound up to logarithmic factors.  Simulations confirm the predicted sub-linear growth in regret.

To highlight that this elevated regret bound is a consequence of rewards depending on the missingness, we study a variant of OTP, namely the Online Cost-sensitive Maximum Entropy Sampling Problem. In OCMESP, rewards depend only on observed tests, and an iterative-elimination algorithm that leverages plug-in entropy estimates can attain a significantly faster $\tilde{\mathcal{O}}(d^3\sigma\sqrt {T})$ regret bound.

\newpage
\begin{APPENDICES}
\section{Proofs of Theorems}

This section provides the proofs for the Theorems and their supporting Lemmas. We note that notations introduced in the main body are not reintroduced here but are held as is. Unless otherwise stated, we use $\|A\|$ to denote the spectral norm of a matrix $A$.

\subsection{Proof of Theorem 1}
\label{apx:single_lb}
\begin{lemma}[Lemma \ref{thm:fail_lb} in the main body]
Suppose $N\leq \frac{1}{8\epsilon^2}$. For any $g$, there exists at least one problem instance $\mathcal{I}_i$ such that
\[P_i(g(\psi)=i)< \frac{3}{4}.\]
\end{lemma}

\begin{proof}{Proof.}
Because we assume the realizations of test outcomes are i.i.d., the KL divergence between distributions $P_1$ and $P_2$ is the summation of individual KL divergence between every pair of Bernoulli random variables (defined as $\text{Ber}(\cdot)$), i.e., $\text{KL}(P_1 \mid P_2) = N \cdot \text{KL}(\text{Ber}(\frac{1+\epsilon}{2}) \mid \text{Ber}(\frac{1-\epsilon}{2}))$. 

For Bernoulli random variables, we have 
\begin{align*}
\text{KL}\left(\text{Ber}(\frac{1+\epsilon}{2})\mid \text{Ber}(\frac{1-\epsilon}{2})\right) &= \frac{1-\epsilon}{2} \cdot \ln \frac{1-\epsilon}{1+\epsilon} + \frac{1+\epsilon}{2} \cdot \ln \frac{1+\epsilon}{1-\epsilon}\\
 &=\epsilon\ln\left(\frac{1+\epsilon}{1-\epsilon}\right)\\
 &=\epsilon\ln\left(1+\frac{2\epsilon}{1-\epsilon}\right)\\
 &\leq \frac{2\epsilon^2}{1-\epsilon}\\
 &\leq 4\epsilon^2
\end{align*}

By Pinker's inequality, for any event $A$, we have 
\begin{align*}
2(P_1(A)-P_2(A))^2 \leq \text{KL}(P_1 \mid P_2) \leq 4N\epsilon^2,
\end{align*} so the distance is upper bounded by $|P_1(A)-P_2(A)| \leq \epsilon\sqrt{2N}< \frac{1}{2}$. 

For contradiction, we assume there exists a policy $\pi$ that can output the correct action with a probability of at least $\frac{3}{4}$. Mathematically, this means $P_1(g(\psi) =1)\geq \frac{3}{4}$ and $P_2(g(\psi) = 2)\geq \frac{3}{4}$. Then it follows that we simultaneously have $P_1(g(\psi) =1)\geq \frac{3}{4}$ and $P_2(g(\psi) = 1) < \frac{1}{4}$, implying $|P_1(g(\psi) =1)-P_2(g(\psi) =1)|>\frac{1}{2}$, which is a contradiction.
\end{proof}

\begin{theorem}[Theorem \ref{thm:lb_single} in the main body]
Fix time horizon $T$. For any algorithm $\pi$, there exists a problem instance such that $\mathbb{E}[R_T]$ is at least $\Omega(T^{\frac{2}{3}})$.
\end{theorem}
\begin{proof}{Proof.}
Fix an arbitrary algorithm $\pi$. Consider a uniform distribution over problem instances $\mathcal{I}_1$ and $\mathcal{I}_2$ where the probabilities are $\frac{1}{2}$ each. We let the test cost be $\frac{3}{4}$. First, note that the optimal expected cost for both problems is $\frac{1-\epsilon}{2}$, so testing is clearly the suboptimal choice, which gives a regret of $\frac{3}{4} - \frac{1-\epsilon}{2} = \frac{1+2\epsilon}{4} > \frac{1}{4}$.

On the one hand, by Lemma \ref{thm:fail_lb}, if $N\leq \frac{1}{8\epsilon^2}$, we know that the expected probability of making a suboptimal decision is at least $\frac{1}{4}$ on at least one of the instances. As such, the expected cost of choosing action $1$ or $2$ is at least $\frac{\epsilon}{8}$ for a single action. Hence, for the uniform distribution over problem instances, the probability of making a suboptimal decision is at least $\frac{1}{4}\times \frac{1}{2}=\frac{1}{8}$. In case of a suboptimal decision, the regret is at least $\frac{1+\epsilon}{2}-\frac{1-\epsilon}{2}=\epsilon$, so the total regret will be at least $\frac{\epsilon (T-N)}{8} =\Omega(\epsilon T - \frac{1}{\epsilon})$. On the other hand, if $N\geq \frac{1}{8\epsilon^2}$, then we suffer a regret of $\frac{1+2\epsilon}{4}N\geq \frac{1+2\epsilon}{4} \cdot \frac{1}{4\epsilon^2} = \Omega(\frac{1}{\epsilon^2})$. Therefore, the total regret is the minimum of them. By choosing $\epsilon=T^{-\frac{1}{3}}$, we have $\Omega(\min(\epsilon T - \frac{1}{\epsilon},\frac{1}{\epsilon^2}))=\Omega(T^{\frac{2}{3}})$. 

\hfill $\square$  
\end{proof}

\subsection{Proof of Theorem 2}
\label{apx:lb}

\begin{theorem}[Theorem \ref{thm:lb} in the main body]
Fix any time horizon $T$ and any $\kappa>0$. For any algorithm, there exists a problem instance with the support of $\mathcal{P}$ being size $\frac{|\mathcal{P}|}{2}$ such that $\mathbb{E}[R_T]$ is at least $\Omega(|\mathcal{P}|^{\frac{1}{3}}T^{\frac{2}{3}})$.
\end{theorem}
\begin{proof}{Proof.}
The cumulative regret of this MDP is the summation of all cumulative regrets for every realization of $x_2$. We denote $T_i$ as the total number of times $x_2=i$ during the $T$ horizon and $N_i$ as the total number of times test $1$ is chosen when $x_2=i$. Clearly, $N_i\leq T_i$ and $\sum_i T_i = T$. By Theorem \ref{thm:fail_lb}, we know that the cumulative regret is at least $\Omega(T_i^{\frac{2}{3}})$ conditioning on any $x_2=i$. Therefore, the expected total cumulative regret is at least $\mathbb{E}_{T_1,\dots, T_{\frac{|\mathcal{P}|}{2}}}[\sum_{i\in [|\mathcal{P}|/2]} \Omega (T_i^{\frac{2}{3}})] = \sum_{i\in [|\mathcal{P}|/2]}\mathbb{E}_{T_i}[\Omega (T_i^{\frac{2}{3}})] = \sum_{i\in [|\mathcal{P}|/2]}\Omega (\mathbb{E}[T_i]^{\frac{2}{3}}) = \Omega (|\mathcal{P}|(\frac{T}{|\mathcal{P}|})^{\frac{2}{3}}) = \Omega(|\mathcal{P}|^{\frac{1}{3}}T^{\frac{2}{3}})$,
where the second equality is by Jensen's Inequality. 

\hfill $\square$  
\end{proof}

\subsection{Proof of Theorem 3}
\label{apx:etcD}
\begin{lemma}[$L_1$ Deviation of Empirical Distributions, Theorem 2.1 in \citealp{weissman2003inequalities}]
\label{thm:L1bound}
Let $P$ be a probability distribution on a finite set 
\(
\mathcal{X} = \{1,2,\ldots,k\},
\)
and let $X_1,X_2,\dots,X_n$ be i.i.d.\ random variables drawn from $P$. Denote by $\hat{P}_n$ the corresponding empirical distribution. Then, for any $\epsilon>0$, 
\[
\mathbb{P}\Bigl(\|\hat{P}_n-P\|_1 \ge \epsilon\Bigr) \le (2^k-2)\,\exp\Bigl(-n\epsilon^2/2\Bigr).
\]
\end{lemma}

\begin{lemma}[$L_1$ Deviation of All Empirical Posterior Distributions]
\label{thm:L1bound_all}
Denote the $P^{\bm{s}}$ as the probability measure for the posterior distribution $\bm{x}\mid\bm{s}$ and $\hat{P}^{\bm{s}}$ the corresponding empirical distribution. Denote $N(\bm{s})$ as the samples used to construct the empirical distribution $\hat{P}^{\bm{s}}$, and $k(\bm{s})$ is the support size of distribution $\bm{x}\mid \bm{s}$. Then, for a probability of at least $1-\delta$, 
\[
\|\hat{P}^{\bm{s}}-P^{\bm{s}}\|_1 \lesssim  \sqrt{\frac{k(\bm{s})}{N(\bm{s})}}.
\]
\end{lemma}
\begin{proof}{Proof.}
This lemma holds by imposing union bounds on Lemma \ref{thm:L1bound}. We claim that there are at most $|\mathcal{P}|^2$ number of posterior distributions to estimate, and hence allocating $2\delta/|\mathcal{P}|^2$ to each state will give the desired outcome. 

To show that there are at most $|\mathcal{P}|^2$ posterior distributions to estimate, we note that for any state $\bm{s}$ that is worth estimating, there exist at least two different realizations $\bm{x}',\bm{x}''\in \mathcal{P}$ such that $\bm{x}'\overset{\Re}{=}\bm{s}$ and $\bm{x}''\overset{\Re}{=} \bm{s}$. Moreover, we can see that for any two states $\bm{s}',\bm{s}''$, if we have $\{\bm{x}\in \mathcal{P}|\bm{x}\overset{\Re}{=}\bm{s}'\}=\{\bm{x}\in \mathcal{P}|\bm{x}\overset{\Re}{=}\bm{s}''\}$, then the posterior distribution $\bm{x}\mid\bm{s}'$ equals $\bm{x}\mid\bm{s}''$. Hence, if we collect one $\bm{s}$ for every pair $\bm{x}',\bm{x}''$ and include the initial state of all \texttt{NA}s, the resulting set of states covers all the states that have to be estimated. By this construction, we can see that there are at most $|\mathcal{P}|(|\mathcal{P}|-1)/2+1 \leq \frac{|\mathcal{P}|^2}{2}$ elements in the resulting set, which ends the proof. 
\end{proof}

\begin{lemma}[Hoeffding’s inequality, Theorem 2.2.5 in \citealp{vershynin2018high}]
\label{thm:hoef}
Let $X_1,\dots,X_n$ be independent Bernoulli random variables with parameters $u$. Then, for any $\epsilon > 0$, we have 
 \[\mathbb{P}\left(\left|\sum_i^{(N)} X_i -u\right|\geq \epsilon\right) \leq \exp(-2\epsilon^2/N).\]
\end{lemma}

\begin{lemma}[Deviation of $N(\bm{s})$]
\label{thm:samplesize}
Denote $N(\bm{s})$ as the number of samples in the dataset that contains $\bm{s}$. Then, for a probability of at least $1-\delta$, 
\[
\|\mathbb{P}(\bm{x}\overset{\Re}{=}\bm{s})\cdot N \|\lesssim N(\bm{s})
\]
\end{lemma}
\begin{proof}{Proof.}
    This is a direct result of Lemma \ref{thm:L1bound}. As we have already seen in the proof of Lemma \ref{thm:L1bound_all}, there are at most $|\mathcal{P}|^2$ number of posterior distributions to estimate, and hence allocating $2\delta/|\mathcal{P}|^2$ to each state should give the desired outcome. We let $\exp(-2\epsilon^2/N) = 2\delta/|\mathcal{P}|^2$ and solve for $\epsilon$. This gives
    \(\mathbb{P}(\bm{x}\overset{\Re}{=}\bm{s})\cdot N +\sqrt{N} \lesssim N(\bm{s}),\)
    which ends the proof.
\end{proof}

\begin{theorem}[Theorem \ref{thm:etcD} in the main body]
    Given time horizon $T$, number of tests $d$, and bounded costs $\{c_i\}_{i\in [d]}$ and reward function $f$, for any fixed discrete distribution $\mathcal{P}$ with finite support, with a probability of at least $1-\delta$, the cumulative regret of running Algorithm \ref{alg:etcD} is at most $\tilde{O}(d|\mathcal{P}|^{\frac{1}{3}}T^{\frac{2}{3}})$.
\end{theorem}

\begin{proof}{Proof.}
    To prove the convergence of cumulative regret, the main challenge lies in controlling the instantaneous regret in commit steps. Ultimately, we need to control the estimation error $|v(\bm{s})-\hat{v}(\bm{s})|$ since our instantaneous regret will be bounded by it.

    We claim that $\Delta v(\bm{s}) \lesssim \sqrt{\frac{k(\bm{s})}{N(\bm{s})}}$, where $k(\bm{s})$ is the support size of distribution $\bm{x}\mid \bm{s}$. To prove this, we use an induction argument. We first start from the base case where there are no $\texttt{NA}$ in $\bm{s}$. Clearly, $\Delta v(\bm{s})=0$ because no estimation is needed. Now, suppose $\Delta v(\bm{s}) \lesssim \sqrt{\frac{k(\bm{s})}{N(\bm{s})}}$ holds true for all $\bm{s}$ with $n$ number of $\texttt{NA}$, we want to show that it is true also for all $\bm{s}$ with $n+1$ number of $\texttt{NA}$.
    
    Because $v(\bm{s})$ is estimated using $Q(\bm{s},i)$ and $\mathbb{E}(\bm{s},y)$, we have $\Delta v(\bm{s}) \leq \max_i{\Delta Q(\bm{s},i)} \vee \max_y \Delta \mathbb{E}(\bm{s},y)$. We start by controlling $\Delta \mathbb{E}(\bm{s},y)$. For all $\bm{s}$ and $y$, we have
    \begin{align*}
        \Delta \mathbb{E}(\bm{s},y)&\triangleq |\mathbb{E}[r(\bm{s},y)]-\hat{\mathbb{E}}[r(\bm{s},y)]| \\
        &= \left|\sum_{\bm{x}} f(\bm{x},y)\cdot \mathbb{P}(\bm{x} \mid \bm{s}) - \sum_{\bm{x}} f(\bm{x},y)\cdot \hat{\mathbb{P}}(\bm{x} \mid \bm{s})\right|\\
        &= \sum_{\bm{x}} f(\bm{x},y) \|\mathbb{P}(\bm{x} \mid \bm{s}) - \hat{\mathbb{P}}(\bm{x} \mid \bm{s})\| \\
        &\leq |f_{\max}| \cdot \|P^{\bm{s}} - \hat{P}^{\bm{s}}\|_1 \\
        & \lesssim \sqrt{\frac{k(\bm{s})}{N(\bm{s})}}.
    \end{align*}
    
    Now, it leaves for us to control the estimation error of $Q(\bm{s},i)$.  Following Eq. (2), we have 
    \begin{align*}
        &\Delta Q(\bm{s},i) \\
        &= \left|\sum_{x_i} v( \bm{s}\oplus (i,x_i)) P^{\bm{s}}(x_i) - \hat{v}( \bm{s}\oplus (i,x_i))\hat{P}^{\bm{s}}(x_i))\right|\\
        &\leq \sum_{x_i} \left|v( \bm{s}\oplus (i,x_i))P^{\bm{s}}(x_i) - (v( \bm{s}\oplus (i,x_i)) + \Delta v( \bm{s}\oplus (i,x_i)))(P^{\bm{s}}(x_i)- \hat{P}^{\bm{s}}(x_i))\right|\\
        &\leq \sum_{x_i} |\Delta v( \bm{s}\oplus (i,x_i))| P^{\bm{s}}(x_i) +\sum_{x_i} v( \bm{s}\oplus (i,x_i))\Delta P^{\bm{s}}(x_i) +\sum_{x_i}  \Delta v( \bm{s}\oplus (i,x_i)) \Delta P^{\bm{s}}(x_i)
    \end{align*}
    For the first term, we have 
    \begin{align*}
        \sum_{x_i} |\Delta v( \bm{s}\oplus (i,x_i))| P^{\bm{s}}(x_i) \lesssim \sum_{x_i} \sqrt{\frac{k(\bm{s}\oplus (i,x_i))}{N( \bm{s}\oplus (i,x_i))}} P^{\bm{s}}(x_i) \\= \sum_{x_i}\sqrt{\frac{k(\bm{s}\oplus (i,x_i))}{N( \bm{s}\oplus (i,x_i))/P^{\bm{s}}(x_i)}}\sqrt{P^{\bm{s}}(x_i)}.
    \end{align*}
    Note that $\mathbb{E}_{x_i}[N( \bm{s}\oplus (i,x_i))/P^{\bm{s}}(x_i)] = \mathbb{E}_{x_i}[N( \bm{s})\cdot P^{\bm{s}}(x_i)/P^{\bm{s}}(x_i)] = N(\bm{s})$. By Lemma \ref{thm:samplesize}, we can upper bound the above equation using the following term:
    \[\sum_{x_i}\sqrt{\frac{k(\bm{s}\oplus (i,x_i))}{N(\bm{s})}}\sqrt{P^{\bm{s}}(x_i)} \leq \sqrt{\frac{\sum_{x_i} k(\bm{s}\oplus (i,x_i)) \sum_{x_i} P^{\bm{s}}(x_i) d}{N(\bm{s})}} = \sqrt{\frac{k(\bm{s})}{N(\bm{s})}},\]
    where the first inequality is due to Cauchy-Schwarz and the second inequality is due to $k(\bm{s}) = \sum_{x_i} k(\bm{s}\oplus (i,x_i))$.
    For the second term, it is also upper bounded by $\tilde{O}(\sqrt{\frac{k(\bm{s})}{N(\bm{s})}})$ following the same argument of how we control $\Delta \mathbb{E}(\bm{s},y)$.

    For the third term, we note that this is an integration of a higher-order term at $O(\frac{1}{N(\bm{s})})$, which will be dominated by the first two terms. To sum up, we have all the three terms upper bounded by $\tilde{O}\left(\sqrt{\frac{k(\bm{s})}{N(\bm{s})}}\right)$, which complete the proof for induction. 

    Now it remains to upper bound instantaneous regrets using $\Delta v(\bm{s})$. For instantaneous regrets at a certain state $\bm{s}$, they can be upper bounded by 
    \(\sum_{\bm{s}} \Delta v(\bm{s}) \mathbb{P}(\text{hit } \bm{s}),\) where $\mathbb{P}(\text{hit } \bm{s})$ is the probability of hitting state $\bm{s}$ at a certain episode. Clearly, we must have $\sum_s \mathbb{P}(\text{hit } \bm{s}) \leq d$ because there are at most $d$ rounds. Now, we bound the expected instantaneous regret as follows:
\begin{align*}
    \sum_{\bm{s}} \Delta v(\bm{s}) \mathbb{P}(\text{hit } \bm{s}) 
    & \lesssim \sum_{\bm{s}} \sqrt{\frac{k(\bm{s})}{N(\bm{s})}}  \mathbb{P}(\text{hit } \bm{s})\\
    & \lesssim \sum_{\bm{s}} \sqrt{\frac{k(\bm{s})}{\mathbb{P}(\text{$\bm{s}$ appears in dataset}) N}}  \mathbb{P}(\text{hit } \bm{s})\\
    & \leq \sum_{\bm{s}:\mathbb{P}(\text{hit } \bm{s}) > 0} \sqrt{\frac{k(\bm{s}) \mathbb{P}(\text{hit } \bm{s})}{N}}\\
    & \leq \frac{1}{N} \sqrt{\sum_{\bm{s}:\mathbb{P}(\text{hit } \bm{s}) > 0} k(\bm{s}) \cdot \sum_{\bm{s}} \mathbb{P}(\text{hit } \bm{s})}.
\end{align*}
The first line is by applying the error bound on $\Delta v(\bm{s})$ that we just derived. The second line is by Lemma \ref{thm:samplesize}. The third line is by the fact that $\mathbb{P}(\text{hit } \bm{s}) \geq \mathbb{P}(\text{$\bm{s}$ appears in dataset})$. This is because $\mathbb{P}(\text{hit } \bm{s}) = \mathbb{P}(\bm{x})\cdot \mathbb{P}(\text{$\bm{s}$ appears in dataset}) = \sum_{\bm{x}}\mathbb{P}(\bm{x}) \cdot \mathbb{P}(\bm{x} \overset{\Re}{=}\bm{s})$. The fourth line follows Cauchy-Schwarz. 

Because we use a deterministic policy here, it means that for any given $\bm{s}$, there is at most one $i$ such that $\mathbb{P}(\bm{s}\oplus(i,x_i))$ is not zero. In addition, we know that $\sum_{x_i} k(\bm{s}\oplus(i,x_i)) = k(\bm{s})$ for any given $i$. We denote $\#\texttt{NA}(\bm{s})$ as the number of missing values in $\bm{s}$. When $\#\texttt{NA}(\bm{s})=d$, we trivially have $k(\bm{s})=|\mathcal{P}|$. Now, for any $\#\texttt{NA}(\bm{s})=d'$, we know that \[\sum_{\bm{s}:\mathbb{P}(\text{hit } \bm{s}) > 0,\#\texttt{NA}(\bm{s})=d'} k(\bm{s}) \geq \sum_{\bm{s}:\mathbb{P}(\text{hit } \bm{s}) > 0,\#\texttt{NA}(\bm{s})=d'} \max_i\sum_{x_i} k(\bm{s}\oplus(i,x_i)) \geq \sum_{\bm{s}:\mathbb{P}(\text{hit } \bm{s}) > 0,\#\texttt{NA}(\bm{s})=d'+1} k(\bm{s}).\] Therefore, we must have $\sum_{\bm{s}:\mathbb{P}(\text{hit }\bm{s})>0}k(\bm{s})\leq d|\mathcal{P}|$, and hence 
\[\sum_{\bm{s}} \Delta v(\bm{s}) \mathbb{P}(\text{hit } \bm{s}) \lesssim \frac{1}{N}\sqrt{d|\mathcal{P}|\cdot d} = d\sqrt{\frac{|\mathcal{P}|}{N}}.\]

    Because this accumulates $T-N$ times for every exploitation round, we have the total regret for the exploitation phase being $\tilde{O}\left(dT\sqrt{\frac{|\mathcal{P}|}{N}}\right)$. The total regret for the exploration phase is $O(dN)$. Now, let $N = |\mathcal{P}|^{\frac{1}{3}}T^{\frac{2}{3}}$ and complete the proof.

     \hfill $\square$  \end{proof}

\subsection{Proof of Theorem 4}
\label{apx:etcG}

For the following proofs, we denote $\lambda_{\max}$ and $\lambda_{\min}$ as the maximum and the minimum eigenvalues of the covariance matrix $\Sigma$.

\begin{lemma}[Tail Bound for Mean Estimation]
For $n$ i.i.d samples $\bm{x}^{(t)}$, with a probability of at least $1-\delta$, the estimated mean vector $\hat{\bm{\mu}} = \frac{1}{N}\sum_{t=1}^{(N)} \bm{x}^{(t)}$ satisfy
    \[\|\hat{\bm{\mu}} - \bm{\mu}\| \lesssim \sqrt{\lambda_{\max}\frac{d}{N}}\]
\end{lemma}
\begin{proof}{Proof.}
    The proof is standard following Markov's equality since the variance of $\|\epsilon\|^2$ is upper bounded by $\lambda_{\max}\frac{d}{N}$.
\end{proof}

\begin{lemma}[Tail Bound for Covariance Estimation, Exercise 4.7.3 in \citealp{vershynin2018high}]
For $n$ i.i.d samples $\bm{x}^{(t)}$, with a probability of at least $1-\delta$, the estimated covariance $\hat{\Sigma}^{(N)} = \frac{1}{N}\sum_{t=1}^{(N)} \bm{x}^{(t)}{\bm{x}^{(t)}}^\intercal$ satisfy
    \[\|\hat{\Sigma}^{(N)} - \Sigma\| \leq CK_1^2\lambda_{\max}\left(\sqrt{\frac{d+\ln(2/\delta)}{N}}+\frac{d+\ln(2/\delta)}{N}\right) \lesssim \lambda_{\max}\sqrt{\frac{d}{N}}\]
\end{lemma}

\begin{lemma}[$L_1$ Deviation of Gaussian Distribution]
\label{thm:L1Gaus}
Let $\bm{\mu}_1,\bm{\mu}_2\in\mathbb{R}^d$ and $\Sigma_1,\Sigma_2$ be $d\times d$ positive-definite matrices. Let $\lambda_{\min}$ be the minimum eigenvalue of $\Sigma_2$. If \( \|\Sigma_1-\Sigma_2\| \leq \lambda_{\min}/2\), we have
\[\|N(\bm{\mu}_1,\Sigma_1),N(\bm{\mu}_2,\Sigma_2)\|_1 \le \frac{\sqrt{d}}{2\lambda_{\min}}\|\Sigma_1-\Sigma_2\| + \frac{1}{2\sqrt{\lambda_{\min}}}\|\bm{\mu}_1-\bm{\mu}_2\|.\]
\end{lemma}

\begin{proof}{Proof}
Because \( \|\Sigma_1-\Sigma_2\| \leq \lambda_{\min}/2\), we have 
 \[\|\Sigma_2^{-1/2}\Sigma_1\Sigma_2^{-1/2} - I\| \leq \|\Sigma_2^{-1/2} (\Sigma_1-\Sigma_2) \Sigma_2^{-1/2}\| \leq \|\Sigma_2^{-1/2}\|\cdot \|\Sigma_1-\Sigma_2\| \cdot \|\Sigma_2^{-1/2}\| \leq \frac{\lambda_{\min}}{2} \cdot\frac{1}{\lambda_{\min}} \leq \frac{1}{2}.\]
Hence, using Lemma 2.9 in \cite{ashtiani2020near}, we can upper bounded the KL divergence using
\[
\text{KL}\Big(N(\bm{\mu}_1\mid \Sigma_1) \mid N(\bm{\mu}_2,\Sigma_2)\Big)
\le \frac{1}{2} \left(d\Big\|\Sigma_2^{-1/2}\Sigma_1\Sigma_2^{-1/2} - I\Big\|^2
+ (\bm{\mu}_2 - \bm{\mu}_1)^\top \Sigma_2^{-1} (\bm{\mu}_2 - \bm{\mu}_1)\right).
\]
Since $(\bm{\mu}_2 - \bm{\mu}_1)^\top \Sigma_2^{-1} (\bm{\mu}_2 - \bm{\mu}_1) = \|\Sigma_2^{-1/2}(\bm{\mu}_1-\bm{\mu}_2)\|^2_2$, by Pinsker’s inequality, we have
\begin{align*}
\|N(\bm{\mu}_1,\Sigma_1) - N(\bm{\mu}_2,\Sigma_2)\|_1
\le \frac{1}{2} \sqrt{d\Big\|\Sigma_2^{-1/2}\Sigma_1\Sigma_2^{-1/2} - I\Big\|^2
+ (\bm{\mu}_2 - \bm{\mu}_1)^\top \Sigma_2^{-1} (\bm{\mu}_2 - \bm{\mu}_1)}\\
\le \frac{1}{2} \sqrt{d}\Big\|\Sigma_2^{-1/2}\Sigma_1\Sigma_2^{-1/2} - I\Big\|+ \frac{1}{2}\|\Sigma_2^{-1/2}(\bm{\mu}_1-\bm{\mu}_2)\|.
\end{align*}
Now, by basic matrix algebra, we have 
\[\frac{1}{2}\sqrt{d}\Big\|\Sigma_2^{-1/2}\Sigma_1\Sigma_2^{-1/2} - I\Big\|+\frac{1}{2}\|\Sigma_2^{-1/2}(\bm{\mu}_1-\bm{\mu}_2)\|\le \frac{\sqrt{d}}{2\lambda_{\min}}\|\Sigma_1-\Sigma_2\| + \frac{1}{2\sqrt{\lambda_{\min}}}\|\bm{\mu}_1-\bm{\mu}_2\|,\]
which ends the proof.
\end{proof}

\begin{lemma}[$L_1$ Deviation of Gaussian Posteriors]
Let $\bm{x}^{(1)},\dots,\bm{x}^{(N)}$ be i.i.d. samples drawn from multivariate Gaussian distribution $\mathcal{N}(\mu,\Sigma)$ where the maximum and minimum eigenvalue of the covariance matrix $\Sigma$ are $\lambda_{\max}$ and $\lambda_{\min}$, respectively. Let $\hat{\Sigma} = \frac{1}{N} \sum_{t=1}^{(N)}\bm{x}^{(t)}{\bm{x}^{(t)}}^\intercal$ and $\hat{\bm{\mu}} = \frac{1}{N} \sum_{t=1}^{(N)}\bm{x}^{(t)}$. Now, for any $\bm{s}$, we denote the true posterior distribution $\bm{x}\mid\bm{s}$ as $P^{\bm{s}}$ and the estimated posterior distribution using the plug-in estimator $\hat{\Sigma}$ and $\hat{\bm{\mu}}$ as $\hat{P}^{\bm{s}}$. Then, for sufficiently large $n\gtrsim d\sigma^2$ such that 
\(\|\hat{\bm{\mu}}-\bm{\mu}\|\lesssim \lambda_{\min}/2\) and \(\|\Sigma-\hat{\Sigma}\| \leq \lambda_{\min}/2\), there holds
\[
\|P^{\bm{s}}-\hat{P}^{\bm{s}}\|_1 \lesssim \frac{\lambda_{\max}^3d}{\lambda_{\min}^3\sqrt{N}} + \frac{\lambda_{\max}^{1.5}}{\lambda_{\min}^{2.5}} \sqrt{\frac{d}{N}}\|\bm{x}-\bm{\mu}\|.
\]
\end{lemma}

\begin{proof}{Proof.}
To set up notations, we let
\[
\bm{x} = \begin{pmatrix} \bm{x}[a] \\ \bm{x}[b] \end{pmatrix} \sim \mathcal{N}\!\left(\begin{pmatrix} \bm{\mu}[a] \\ \bm{\mu}[b] \end{pmatrix},\, \Sigma\right)
\]
with
\[
\Sigma = \begin{pmatrix} \Sigma_{aa} & \Sigma_{ab} \\ \Sigma_{ba} & \Sigma_{bb} \end{pmatrix}.
\]

Here, $a$ denotes the set of unobserved indices, and $b$ denotes the set of observed indices. The posterior distribution $\bm{x}\mid\bm{s}$ is essentially the distribution of $\bm{x}[a]$ given $\bm{x}[b]$. Note that the choice of $a$ and $b$ is arbitrary for this proof, so the result generalizes to any posterior distribution. 

By our assumption, we have estimates \(\hat{\bm{\mu}}\) and \(\hat{\Sigma}\) for the mean and covariance satisfying
\[
\|\hat{\bm{\mu}}-\bm{\mu}\|\lesssim \epsilon_1 \leq \frac{\lambda_{\min}}{2},\quad \|\hat{\Sigma}-\Sigma\|\lesssim \epsilon_2 \leq \frac{\lambda_{\min}}{2}.
\]
The true posterior mean is
\[
\bm{\mu}_{a|b} = \bm{\mu}[a] + \Sigma_{ab}\Sigma_{bb}^{-1}(\bm{x}[b]-\bm{\mu}[b]),
\]
and the plug–in estimator is
\[
\hat{\bm{\mu}}_{a|b} = \hat{\bm{\mu}}[a] + \hat{\Sigma}_{ab}\hat{\Sigma}_{bb}^{-1}(\bm{x}[b]-\hat{\bm{\mu}}[b]).
\]
Similarly, define the posterior covariance matrix (and its estimator) as
\[
\Sigma_{a|b} = \Sigma_{aa} - \Sigma_{ab} \Sigma_{bb}^{-1} \Sigma_{ba},
\]
\[
\hat{\Sigma}_{a|b} = \hat{\Sigma}_{aa} - \hat{\Sigma}_{ab} \hat{\Sigma}_{bb}^{-1} \hat{\Sigma}_{ba}.
\]
For convenience, denote the error in each block by
\[
\Delta_{ij} = \hat{\Sigma}_{ij} - \Sigma_{ij}, \quad \text{for } ij\in\{aa,\,ab,\,ba,\,bb\}.
\]

To leverage the results in Lemma \ref{thm:L1Gaus}, it suffices to bound the error in the posterior mean and the covariance, respectively.

\noindent\textbf{Posterior Mean Error Bound:}

By definition,
\[
\hat{\bm{\mu}}_{a|b}-\bm{\mu}_{a|b} = (\hat{\bm{\mu}}[a]-\bm{\mu}[a]) + \left[\hat{\Sigma}_{ab}\hat{\Sigma}_{bb}^{-1}(\bm{x}[b]-\hat{\bm{\mu}}[b]) - \Sigma_{ab}\Sigma_{bb}^{-1}(\bm{x}[b]-\bm{\mu}[b])\right].
\]
Noting that
\[
\hat{\Sigma}_{ab}\hat{\Sigma}_{bb}^{-1}(\bm{x}[b]-\hat{\bm{\mu}}[b]) - \Sigma_{ab}\Sigma_{bb}^{-1}(\bm{x}[b]-\bm{\mu}[b])
= \Bigl(\hat{\Sigma}_{ab}\hat{\Sigma}_{bb}^{-1}-\Sigma_{ab}\Sigma_{bb}^{-1}\Bigr)(\bm{x}[b]-\bm{\mu}[b]) - \hat{\Sigma}_{ab}\hat{\Sigma}_{bb}^{-1}(\hat{\bm{\mu}}[b]-\bm{\mu}[b]),
\]
we obtain, by the triangle inequality,
\[
\|\hat{\bm{\mu}}_{a|b}-\bm{\mu}_{a|b}\| \le \|\hat{\bm{\mu}}[a]-\bm{\mu}[a]\| + \|\hat{\Sigma}_{ab}\hat{\Sigma}_{bb}^{-1}-\Sigma_{ab}\Sigma_{bb}^{-1}\|\,\|\bm{x}[b]-\bm{\mu}[b]\| + \|\hat{\Sigma}_{ab}\hat{\Sigma}_{bb}^{-1}\|\,\|\hat{\bm{\mu}}[b]-\bm{\mu}[b]\|,
\]
so we will bond the three terms one by one.

\noindent \textbf{First Term} $\|\hat{\bm{\mu}}[a]-\bm{\mu}[a]\|$: 

By assumption, we have 
\(\|\hat{\bm{\mu}}[a]-\bm{\mu}[a]\|\le \epsilon_1\) and \(\|\hat{\bm{\mu}}[b]-\bm{\mu}[b]\|\le \epsilon_1.\)

\noindent \textbf{Second Term} $\|\hat{\Sigma}_{ab}\hat{\Sigma}_{bb}^{-1}-\Sigma_{ab}\Sigma_{bb}^{-1}\|\,\|\bm{x}[b]-\bm{\mu}[b]\|$:

\medskip

We first decompose this term. To ease notations, define the perturbations
\[
\Delta_{ab} = \hat{\Sigma}_{ab} - \Sigma_{ab}, \qquad \Delta_{bb} = \hat{\Sigma}_{bb} - \Sigma_{bb}.
\]
Then we have
\[
\hat{\Sigma}_{ab}\hat{\Sigma}_{bb}^{-1} - \Sigma_{ab}\Sigma_{bb}^{-1} = \Delta_{ab}\hat{\Sigma}_{bb}^{-1} + \Sigma_{ab}\Bigl(\hat{\Sigma}_{bb}^{-1} - \Sigma_{bb}^{-1}\Bigr).
\]

We first bound $\|\Delta_{ab}\hat{\Sigma}_{bb}^{-1}\|$. Under the assumption that $\epsilon_2 \le \lambda_{\min}/2$, using Banach's Lemma gives
\[
\|\hat{\Sigma}_{bb}^{-1}\| \leq \|\Sigma_{bb}^{-1}\|\cdot\|(I+\Sigma_{bb}^{-1}\Delta_{bb})^{-1}\| \le \frac{\|\Sigma_{bb}^{-1}\|}{1-\|\Sigma_{bb}^{-1}\Delta_{bb}\|} \le \frac{1}{1/\|\Sigma_{bb}^{-1}\|-\|\Delta_{bb}\|} \le \frac{1}{\lambda_{\min}(\Sigma_{bb})-\epsilon_2} \le \frac{2}{\lambda_{\min}}.
\]
Since $\|\Delta_{ab}\| \le \|\hat{\Sigma}-\Sigma\| \le \epsilon_2$ and, it follows that
\[
\|\Delta_{ab}\hat{\Sigma}_{bb}^{-1}\| \le \|\Delta_{ab}\|\,\|\hat{\Sigma}_{bb}^{-1}\| \le \epsilon_2 \cdot \frac{2}{\lambda_{\min}} = \frac{2\epsilon_2}{\lambda_{\min}}.
\]

Second, we bound $\|\Sigma_{ab}(\hat{\Sigma}_{bb}^{-1} - \Sigma_{bb}^{-1})\|$.
By some algebraic manipulation, we have
\[
\hat{\Sigma}_{bb}^{-1} - \Sigma_{bb}^{-1} = -\Sigma_{bb}^{-1}\Delta_{bb}\hat{\Sigma}_{bb}^{-1},
\]
so that
\[
\|\hat{\Sigma}_{bb}^{-1} - \Sigma_{bb}^{-1}\| \le \|\Sigma_{bb}^{-1}\|\,\|\Delta_{bb}\|\,\|\hat{\Sigma}_{bb}^{-1}\| \le \frac{1}{\lambda_{\min}}\,\epsilon_2\,\frac{2}{\lambda_{\min}} = \frac{2\epsilon_2}{\lambda_{\min}^2}.
\]
Thus,
\[
\|\Sigma_{ab}(\hat{\Sigma}_{bb}^{-1} - \Sigma_{bb}^{-1})\| \le \|\Sigma_{ab}\| \cdot \frac{2\epsilon_2}{\lambda_{\min}^2} \le \|\Sigma\| \cdot \frac{2\epsilon_2}{\lambda_{\min}^2} = \frac{2\lambda_{\max}\epsilon_2}{\lambda_{\min}^2}.
\]

Combining the two terms, we have
\[
\|\hat{\Sigma}_{ab}\hat{\Sigma}_{bb}^{-1} - \Sigma_{ab}\Sigma_{bb}^{-1}\| \le \frac{2\epsilon_2}{\lambda_{\min}} + \frac{2\lambda_{\max}\epsilon_2}{\lambda_{\min}^2}.
\]

\noindent \textbf{Thrid Term} $\|\hat{\Sigma}_{ab}\hat{\Sigma}_{bb}^{-1}\|\|\hat{\bm{\mu}}[b]-\bm{\mu}[b]\|$: Using the triangle inequality, we can decompose 
\[
\|\hat{\Sigma}_{ab}\hat{\Sigma}_{bb}^{-1}\| \le \|\Sigma_{ab}\Sigma_{bb}^{-1} - \hat{\Sigma}_{ab}\hat{\Sigma}_{bb}^{-1}\| + \|\Sigma_{ab}\Sigma_{bb}^{-1} \| \le \frac{2\epsilon_2}{\lambda_{\min}} + \frac{2\lambda_{\max}\epsilon_2}{\lambda_{\min}^2} + \frac{\lambda_{\max}}{\lambda_{\min}}.
\]
Combining $\|\hat{\bm{\mu}}[b]-\bm{\mu}[b]\|\le \epsilon_1$, we have
\[
\|\hat{\Sigma}_{ab}\hat{\Sigma}_{bb}^{-1}\|\|\hat{\bm{\mu}}[b]-\bm{\mu}[b]\| \le \epsilon_1(\frac{2\epsilon_2}{\lambda_{\min}} + \frac{2\lambda_{\max}\epsilon_2}{\lambda_{\min}^2} + \frac{\lambda_{\max}}{\lambda_{\min}}) \lesssim \frac{\lambda_{\max}}{\lambda_{\min}}\epsilon_1.
\]

\medskip

Hence, collecting the errors from each term, we conclude that 
\[
\|\hat{\bm{\mu}}_{a|b}-\bm{\mu}_{a|b}\| \lesssim  \frac{\lambda_{\max}\epsilon_1}{\lambda_{\min}^2}\,\|\bm{x}[b]-\bm{\mu}[b]\| = \frac{\lambda_{\max}^{1.5}}{\lambda_{\min}^{2}}\sqrt{\frac{d}{N}}\|\bm{x}[b]-\bm{\mu}[b]\|.
\]

\noindent\textbf{Posterior Variance Error Bound:}

By definition of the posterior variance error,
\[
\left(\hat{\Sigma}_{aa} - \Sigma_{aa}\right) - \Bigl(\hat{\Sigma}_{ab}\hat{\Sigma}_{bb}^{-1}\hat{\Sigma}_{ba} - \Sigma_{ab}\Sigma_{bb}^{-1}\Sigma_{ba}\Bigr).
\]
We can rewrite the second term as
\[
\hat{\Sigma}_{ab}\hat{\Sigma}_{bb}^{-1}\hat{\Sigma}_{ba} - \Sigma_{ab}\Sigma_{bb}^{-1}\Sigma_{ba} 
= (\hat{\Sigma}_{ab} - \Sigma_{ab})\,\hat{\Sigma}_{bb}^{-1}\hat{\Sigma}_{ba} 
+ \Sigma_{ab}\Bigl(\hat{\Sigma}_{bb}^{-1} - \Sigma_{bb}^{-1}\Bigr)\hat{\Sigma}_{ba}
+ \Sigma_{ab}\Sigma_{bb}^{-1}\,(\hat{\Sigma}_{ba}-\Sigma_{ba}).
\]
Thus, by the triangular inequality, we can upper bound the posterior variance error with
\[
\|\Delta_{aa}\| +\|\Delta_{ab}\hat{\Sigma}_{bb}^{-1}\hat{\Sigma}_{ba}\| + \|\Sigma_{ab}\Bigl(\hat{\Sigma}_{bb}^{-1}-\Sigma_{bb}^{-1}\Bigr)\hat{\Sigma}_{ba}\| + \|\Sigma_{ab}\Sigma_{bb}^{-1}\,\Delta_{ba}\|.
\]

\medskip
\noindent \textbf{First term $\|\Delta_{aa}\|$}: 

By assumption, \(
\| \Delta_{aa} \| = \|\hat{\Sigma}_{aa} - \Sigma_{aa}\| \le \epsilon_2.\)

\noindent \textbf{Second term $\|\Delta_{ab}\hat{\Sigma}_{bb}^{-1}\hat{\Sigma}_{ba}\|$}: For the second term, we have
\[
\|\Delta_{ab}\hat{\Sigma}_{bb}^{-1}\hat{\Sigma}_{ba}\| \le \|\Delta_{ab}\|\,\|\hat{\Sigma}_{bb}^{-1}\|\,\|\hat{\Sigma}_{ba}\|.
\]
We have already proved that 
\[
\|\hat{\Sigma}_{bb}^{-1}\| \le \frac{2}{\lambda_{\min}}.
\]
Using the triangular inequality, we have
\[
\|\hat{\Sigma}_{ba}\| \le \|\Sigma_{ba}\| + \|\Delta_{ba}\| \le \lambda_{\max} + \epsilon_2 \le 2\lambda_{\max}.
\]
Since \(\|\Delta_{ab}\|\le \epsilon_2\), it follows that
\[
\|\Delta_{ab}\hat{\Sigma}_{bb}^{-1}\hat{\Sigma}_{ba}\| \le \epsilon_2 \cdot \frac{2}{\lambda_{\min}} \cdot 2\lambda_{\max} = \frac{4\lambda_{\max}\epsilon_2}{\lambda_{\min}}.
\]

\noindent \textbf{Third term $\|\Sigma_{ab}\Bigl(\hat{\Sigma}_{bb}^{-1}-\Sigma_{bb}^{-1}\Bigr)\hat{\Sigma}_{ba}\|$}: For the third term, we have
\[
\|\Delta_{ab}\hat{\Sigma}_{bb}^{-1}\hat{\Sigma}_{ba}\| \le \|\Sigma_{ab}\|\,\|\hat{\Sigma}_{bb}^{-1}-\Sigma_{bb}^{-1}\|\,\|\hat{\Sigma}_{ba}\|.
\]
From previous analysis, we have \(\|\Sigma_{ab}\|\le \lambda_{\max}\), \(\|\hat{\Sigma}_{ba}\|\le 2\lambda_{\max}\) and
\(
\|\hat{\Sigma}_{bb}^{-1}-\Sigma_{bb}^{-1}\| \le \frac{2\epsilon_2}{\lambda_{\min}^2}.
\) Therefore,
\[
\|\Delta_{ab}\hat{\Sigma}_{bb}^{-1}\hat{\Sigma}_{ba}\| \le \lambda_{\max}\cdot\frac{2\epsilon_2}{\lambda_{\min}^2}\cdot 2\lambda_{\max} = \frac{4\lambda_{\max}^2\epsilon_2}{\lambda_{\min}^2}.
\]

\noindent \textbf{Fourth Term $\|\Sigma_{ab}\Sigma_{bb}^{-1}\,\Delta_{ba}\|$}: For
the fourth term, we have
\[
\|\Sigma_{ab}\Sigma_{bb}^{-1}\,\Delta_{ba}\| \le \|\Sigma_{ab}\Sigma_{bb}^{-1}\|\,\|\Delta_{ba}\|.
\]
Since
\[
\|\Sigma_{ab}\Sigma_{bb}^{-1}\| \le \|\Sigma_{ab}\|\,\|\Sigma_{bb}^{-1}\| \le \frac{\lambda_{\max}}{\lambda_{\min}},
\]
and \(\|\Delta_{ba}\|\le \epsilon_2\), we obtain
\[
\|\Sigma_{ab}\Sigma_{bb}^{-1}\,\Delta_{ba}\| \le \frac{\lambda_{\max}\epsilon_2}{\lambda_{\min}}.
\]

\medskip

Now we collect the error from the four terms and obtain an error bound on the posterior covariance estimation. By the triangle inequality,
\[
\|\hat{\Sigma}_{a|b} - \Sigma_{a|b}\| \le \|\Delta_{aa}\| + \|\Delta_{ab}\hat{\Sigma}_{bb}^{-1}\hat{\Sigma}_{ba}\| + \|\Delta_{ab}\hat{\Sigma}_{bb}^{-1}\hat{\Sigma}_{ba}\| + \|\Sigma_{ab}\Sigma_{bb}^{-1}\,\Delta_{ba}\|.
\]
Thus, combining our estimates we have:
\[
\|\hat{\Sigma}_{a|b} - \Sigma_{a|b}\| \le \epsilon_2 + \frac{4\lambda_{\max}\epsilon_2}{\lambda_{\min}} + \frac{4\lambda_{\max}^2\epsilon_2}{\lambda_{\min}^2} + \frac{\lambda_{\max}\epsilon_2}{\lambda_{\min}},
\]
which simplifies to
\[
\|\hat\Sigma_{a|b}-\Sigma_{a|b}\|\lesssim \frac{\lambda_{\max}^2}{\lambda_{\min}^2}\epsilon_2 = \frac{\lambda_{\max}^3}{\lambda_{\min}^2}\sqrt{\frac{d}{N}}.
\]

\medskip

Finally, we combine the errors resulting from covariance and the mean deviations,
\[\|P^{\bm{s}}-\hat{P}^{\bm{s}}\|_1 \lesssim\frac{\sqrt{d}}{\lambda_{\min}}\|\hat\Sigma_{a|b}-\Sigma_{a|b}\| + \frac{1}{\sqrt{\lambda_{\min}}}\|\hat{\bm{\mu}}_{a|b}-\bm{\mu}_{a|b}\|\lesssim \frac{\lambda_{\max}^3d}{\lambda_{\min}^3\sqrt{N}} + \frac{\lambda_{\max}^{1.5}}{\lambda_{\min}^{2.5}} \sqrt{\frac{d}{N}}\|\bm{x}-\bm{\mu}\|,\]
which completes the proof.
    
\end{proof}

\begin{theorem}[Theorem \ref{thm:etcG} in the main body]
    Given time horizon $T$, number of tests $d$, and bounded costs $\{c_i\}_{i\in [d]}$ and reward function $f$, for any fixed Gaussian distribution $\mathcal{P}$ of condition number $\sigma$, with a probability of at least $1-\delta$, the cumulative regret of running Algorithm \ref{alg:etcG} is at most $\tilde{O}(d\sigma^2T^{\frac{2}{3}})$.
\end{theorem}

\begin{proof}{Proof.}
    This proof is similar to Theorem \ref{thm:etcD}, so we will use the same set of notations as in Theorem \ref{thm:etcD}. We claim that $\Delta v(\bm{s}) \lesssim \|P^{\bm{s}}-\hat{P}^{\bm{s}}\|_1 = \epsilon$. To prove this, we use an induction argument. We first start from the base case where there are no $\texttt{NA}$ in $\bm{s}$. Again, $\Delta v(\bm{s})=0$ because no estimation is needed. Now, suppose $\Delta v(\bm{s}) \lesssim \epsilon$ holds true for all $\bm{s}$ with $l$ number of $\texttt{NA}$, we want to show that it is true also for all $\bm{s}$ with $l+1$ number of $\texttt{NA}$.
    
    Because $v(\bm{s})$ is estimated using $Q(\bm{s},i)$ and $\mathbb{E}(\bm{s},y)$, we have $\Delta v(\bm{s}) \leq \max_i{\Delta Q(\bm{s},i)} \vee \max_y \Delta \mathbb{E}(\bm{s},y)$. We start by controlling $\Delta \mathbb{E}(\bm{s},y)$. For all $\bm{s}$ and $y$, we have
    \begin{align*}
        \Delta \mathbb{E}(\bm{s},y)&\triangleq |\mathbb{E}[r(\bm{s},y)]-\hat{\mathbb{E}}[r(\bm{s},y)]| \\
        &= \left|\int f(\bm{x},y)\cdot \mathbb{P}(\bm{x} \mid \bm{s})\,d \bm{x} - \int f(\bm{x},y)\cdot \hat{\mathbb{P}}(\bm{x} \mid \bm{s})\,d \bm{x}\right|\\
        &= \int f(\bm{x},y) \|\mathbb{P}(\bm{x} \mid \bm{s}) - \hat{\mathbb{P}}(\bm{x} \mid \bm{s})\| \,d \bm{x} \\
        &\leq |f_{\max}| \cdot \|P^{\bm{s}} - \hat{P}^{\bm{s}}\|_1 \\
        & \lesssim \sqrt{\frac{k(\bm{s})}{N(\bm{s})}}.
    \end{align*}
    
    Now, it leaves for us to control the estimation error of $Q(\bm{s},i)$.  Following Eq. (2), we have 
    \begin{align*}
        &\Delta Q(\bm{s},i) \\
        &= \left|\int v( \bm{s}\oplus (i,x_i)) P^{\bm{s}}(x_i) - \hat{v}( \bm{s}\oplus (i,x_i))\hat{P}^{\bm{s}}(x_i)) dx_i\right|\\
        &\leq \int \left|v( \bm{s}\oplus (i,x_i))P^{\bm{s}}(x_i) - (v( \bm{s}\oplus (i,x_i)) + \Delta v( \bm{s}\oplus (i,x_i)))(P^{\bm{s}}(x_i)- \hat{P}^{\bm{s}}(x_i))\right| dx_i\\
        &\leq \int |\Delta v( \bm{s}\oplus (i,x_i))| P^{\bm{s}}(x_i) dx_i +\int v( \bm{s}\oplus (i,x_i))\Delta P^{\bm{s}}(x_i) dx_i +\int  \Delta v( \bm{s}\oplus (i,x_i)) \Delta P^{\bm{s}}(x_i) dx_i
    \end{align*}
    For the first term, we note that $P^{\bm{s}}$ is a Gaussian pdf, so 
    \begin{align*}
        \int |\Delta v( \bm{s}\oplus (i,x_i))| P^{\bm{s}}(x_i) dx_i \lesssim \int \epsilon P^{\bm{s}}(x_i) dx_i \leq \epsilon.
    \end{align*}
    For the second term, we also have $\int v( \bm{s}\oplus (i,x_i))\Delta P^{\bm{s}}(x_i) dx_i \lesssim \epsilon$ following the same argument of how we control $\Delta \mathbb{E}(\bm{s},y)$.

    For the third term, we note that this is an integration of a higher-order term, which will be dominated by the first two terms. To sum up, we have each term controlled by $\epsilon$, so ultimately $\Delta v(\bm{s}) \lesssim \epsilon$. 
    
    Now it remains to upper bound instantaneous regrets using $\Delta v(\bm{s})$. We can compute this using Tower's law by taking an expectation over the randomness in $\bm{x}$.
\begin{align*}
    \int \Delta v(\bm{s}) p(\text{hit } \bm{s}) \, d\bm{s} &= \int\int \Delta v(\bm{s}) p(\text{hit } \bm{s} \mid \bm{x}) \, d\bm{s}\, p(\bm{x}) d\bm{x}\\
    &\leq \int \max_{\bm{s}: \bm{x}\overset{\Re}{=}s}\Delta v(\bm{s}) p(\bm{x})  d\bm{x}\\
    &\lesssim \int \left(\frac{\lambda_{\max}^3d}{\lambda_{\min}^3\sqrt{N}} + \frac{\lambda_{\max}^2}{\lambda_{\min}^{2.5}} \sqrt{\frac{d}{N}}\|\bm{x}-\bm{\mu}\|\right) p(\bm{x}) d\bm{x}\\
    &\le \frac{\lambda_{\max}^3d}{\lambda_{\min}^3\sqrt{N}} + \frac{\lambda_{\max}^{1.5}}{\lambda_{\min}^{2.5}} \sqrt{\frac{d}{N}}\sqrt{d}\lambda_{\max} \\
    &\le \frac{\sigma^3d}{\sqrt{N}} + \frac{\sigma^{2.5}d}{\sqrt{N}}
    \lesssim \frac{d\sigma^3}{\sqrt{N}}
\end{align*}

Because this accumulates $T-N$ times for every exploitation round, we have total regret for the exploitation phase being $\tilde{O}\left(\frac{d\sigma^3}{\sqrt{N}}\right)$. The total regret for the exploration phase is $O(dN)$. Now, let $N = \sigma^2 T^{\frac{2}{3}}$ and complete the proof.

     \hfill $\square$  \end{proof}

\subsection{Proof of Theorem 5}
\label{apx:itr}

\begin{lemma}[Bernstein’s inequality, Corollary 2.8.3 in \citealp{vershynin2018high}]
\label{thm:bernstein}
Let $X_1,...,X_N$ be independent, zero mean, sub-exponential random variables. Then, for every $t \geq 0$, we have 
\[P\left(\left|\frac{1}{n}\sum_{i=1}^{(n)} X_i\right|\geq \epsilon\right) \leq 2 \exp\left(-c\min \left(\frac{\epsilon^2}{K_2^2},\frac{\epsilon}{K_2}\right)n\right)\]
where $c > 0$ is an absolute constant and $K_2 = \max_i \inf\{t > 0: \mathbb{E}\exp(|X_i|/t) \leq 2\}. $
\end{lemma}

\begin{lemma}
\label{thm:K}
    Given the definition of $K_2 = \max_i \inf\{t > 0: \mathbb{E}\exp(|X_i|/t) \leq 2\}. $ For any $x_i,x_j$ that is multivariate Gaussian, zero mean, with covariance matrix $\Sigma_0 = \begin{bmatrix}\sigma_i^2&\rho\sigma_i\sigma_j\\      \rho\sigma_i\sigma_j&\sigma_j^2
    \end{bmatrix}$, we must have $K_2\leq \frac{8}{3}\sigma_i\sigma_j$.
\end{lemma}
\begin{proof}{Proof.}
Consider $x_i',x_j'$ which are standardized $x_i,x_j$, where $x_i'=\frac{x_i}{\sigma_i}$ and $x_j' = \frac{x_j}{\sigma_j}$. 
    To prove the argument, it suffices to show that
    \[\mathbb{E}\exp\left(\frac{3}{8}\frac{|x_ix_j|}{\sigma_i\sigma_j}\right) = \mathbb{E}\exp\left(\frac{3}{8}|x_ix_j|\right) \leq 2. \]
    Using the fact that $|x_i'x_j'|\leq \frac{x_i'^2+x_j'^2}{2}$, we can upper bound $\mathbb{E}\exp\left(\frac{3}{8}|x_i'x_j'|\right)$ with $\mathbb{E}\exp\left(\frac{3}{8}(x_i'^2+x_j'^2)\right)$.
    By Equation 5.6 in \cite{khuri2009linear}, we have 
    \[\mathbb{E}\exp\left(\frac{3}{16}(x_i'^2+x_j'^2)\right) = \det\left(I- \frac{3}{8}\Sigma_0\right)^{-\frac{1}{2}}.\]
    Now we plug in $\Sigma_0$
    \[\det\left(I- \frac{3}{8}\Sigma_0\right)^{-\frac{1}{2}} = \left(1-\frac{3}{8}\right)^2-\left(-\frac{3}{8}\rho\right)^2 = \frac{25-9\rho^2}{64}.\]
    Since $|\rho|\leq 1$, we have $\frac{25-9\rho^2}{64} \geq \frac{1}{4}$, implying $\mathbb{E}\exp\left(\frac{3}{16}(x_i'^2+x_j'^2)\right) = \det\left(I- \frac{3}{8}\Sigma_0\right)^{-\frac{1}{2}} \leq 2$, which ends the proof.
\end{proof}

\begin{corollary}
\label{thm:coverror}
    For $n$ i.i.d samples that jointly covers $x_i$ and $x_j$, with a probability of at least $1-\delta$, the estimated covariance $\hat{\Sigma}_n[{i,j}] = \frac{1}{n}\sum_{t=1}^{n} x_i^{(t)}x_j^{(t)}$ satisfy
    \[|\hat{\Sigma}_n[{i,j}] - \Sigma[{i,j}]| \leq \frac{8\lambda_{\max}}{3} \cdot \max\left\{\sqrt{\frac{\ln(2/\delta)}{cn}}, \frac{\ln(2/\delta)}{cn}\right\} \lesssim \frac{\lambda_{\max}}{\sqrt{n}}.\]
\end{corollary}
\begin{proof}{Proof.}
    Because $x_i$ and $x_j$ are both Gaussian, their product $x_i,x_j$ is sub-exponential, so the corollary follows Lemma \ref{thm:bernstein}. By Lemma \ref{thm:K}, we can upper bound $K$ with $\frac{8}{3}\cdot \lambda_{\max}$. Hence, we get the desired result.
     \hfill $\square$  \end{proof}

\begin{lemma}
\label{thm:matrixerror}
    With a probability of at least $1-\delta$, for all paris $(i,j)$,  all the estimated covariance $\hat{\Sigma}^{(t)}[{i,j}]$ satisfy
     \[|\hat{\Sigma}^{(t)}[{i,j}] - \Sigma[{i,j}]| \leq \frac{8\lambda_{\max}}{3}\cdot \max\left\{d\sqrt{\frac{\ln(\pi^2 d^2 t^2/\delta)}{ct}}, \frac{d^2(\ln(\pi^2 d^2 t^2/\delta))}{ct}\right\} \lesssim \frac{d\lambda_{\max}}{\sqrt{t}}.\]
\end{lemma}
\begin{proof}{Proof.}
    By corollary \ref{thm:coverror} and a union-bound argument, if we have at least $n$ samples for every pair $(i,j)$, because there are in total $d^2$ pairs of $(i,j)$, we know that with a probability of at least $1-\frac{\delta'}{n^2}$, 
    \[|\hat{\Sigma}_n[{i,j}] - \Sigma[{i,j}]| \leq \frac{8\lambda_{\max}}{3}\cdot \max\left\{\sqrt{\frac{\ln(2d^2n^2/\delta')}{cn}}, \frac{\ln(2d^2n^2/\delta')}{cn}\right\}.\]

    Again by a union-bound argument, for all $n$, with a probability of at least $1-\frac{\pi^2}{6}\delta'$ (since $\frac{\pi^2}{6}\geq \sum\frac{1}{n^2} $), we have 
    \[|\hat{\Sigma}_n[{i,j}] - \Sigma[{i,j}]| \leq \frac{8\lambda_{\max}}{3}\cdot \max\left\{\sqrt{\frac{\ln(2d^2n^2/\delta')}{cn}}, \frac{\ln(2d^2n^2/\delta')}{cn}\right\}.\]
    
    By the selection rule of the iterative elimination algorithm, we should have at least $t/d^2$ samples for every pair $(i,j)$ but at most $t$ samples up to time $t$. As such, for any pair $(i,j)$, we have $\lfloor \frac{t}{d(d-1)} \rfloor \leq n \leq t$. To simplify notations, we use a looser bound of $\frac{t}{3d^2} \leq n \leq t$. Hence, by substituting $\delta  = \frac{\pi^2}{6}\delta'$, with a probability of at least $1-\delta$,
    \[|\hat{\Sigma}^{(t)}[{i,j}] - \Sigma[{i,j}]| \leq \frac{8\lambda_{\max}}{3}\cdot \max\left\{d\sqrt{\frac{\ln(\pi^2 d^2 t^2/\delta)}{ct}}, \frac{d^2(\ln(\pi^2 d^2 t^2/\delta))}{ct}\right\}.\]
     \hfill $\square$  \end{proof}

\begin{lemma}
\label{thm:logdet}
    For any $d\times d$ matrix $A\succ 0$ and $E$ satisfying $3d\|A^{-1}\|\|E\|<1$, we have  
    \[\left|\log \det (A+E) - \log \det(A)\right| \leq 3d\|A^{-1}\|\|E\|,\]
\end{lemma}
\begin{proof}{Proof.}
    By the relative perturbation bound of determinant \citep[Corollary 2.14, ][]{ipsen2008perturbation}, we have
    \[\left|\frac{\det (A) - \det (A+E)}{\det (A)}\right| \leq \left|\|A^{-1}\|\|E\|+1\right|^d-1.\]
    Note that the right-hand side of the inequality is locally linear on $\|E\|$. 
    For $\|A^{-1}\|\|E\|<\frac{1}{3d}$, there is 
    \begin{align*}
        \left|\|A^{-1}\|\|E\|+1\right|^d-1 &\leq \frac{d\|A^{-1}\|\|E\|}{1-d\|A^{-1}\|\|E\|}\\
        &\leq \frac{3}{2}d\|A^{-1}\|\|E\| \leq \frac{1}{2},
    \end{align*}
    where the first inequality is due to Bernoulli's inequality that $(1+a)^d\geq 1+da$ for all $a\geq 0$ and $d\geq 2$, and the second inequality is given by the condition $d\|A^{-1}\|\|E\|<\frac{1}{3}$.

    Because $\left|\frac{\det (A) - \det (A+E)}{\det (A)}\right| \leq \frac{1}{2}$, we know that $\det (A+E)$ must be greater than $0$. More precisely, there is $\frac{\det(A)}{2} \leq \det(A+E) \leq \frac{3\det(A)}{2}$ given that $\det(A)>0$. As such, both $\log \det (A+E)$ and $\log \det(A)$ are well defined, so we do not need extra assumptions on the positive definiteness of matrix $A+E$.
    
    From $\left|\frac{\det (A) - \det (A+E)}{\det (A)}\right| \leq \frac{3}{2}d\|A^{-1}\|\|E\|$, we can see that the relative error on the change of the determinant is locally linear to the change of $\|E\|$.  To upper-bound $\left|\frac{\det (A) - \det (A+E)}{\det (A)}\right|$, observe that
    \begin{align*}
        \left|\log \frac{\det (A+E)}{\det(A)}\right| = \left|\log \left(1+\frac{\det (A+E)-\det(A)}{\det (A)}\right)\right| 
        \leq 2\left|\frac{\det (A) - \det (A+E)}{\det (A)}\right|,
    \end{align*}
    where the last inequality is due to the fact that the linear approximation on the left end gives an over-estimation for the log function, i.e., $\log(1+a)\leq 2a$ for any $|a|\leq \frac{1}{2}$. 

    Finally, combining inequalities together, we get \[\left|\log \det (A+E) - \log \det(A)\right| \leq 2\left|\frac{\det (A) - \det (A+E)}{\det (A)}\right| \leq 3d\|A^{-1}\|\|E\|,\]
    which ends the proof.
     \hfill $\square$  \end{proof}
\begin{lemma}
\label{thm:simple}
    For all $t$ such that $U^{(t)}\leq 1$, for any subset $\bm{s}$, with a probability of at least $1-\delta$, the estimated log determinant using the plug-in estimator satisfies
    \[\left|\log\det \left(\hat{\Sigma}^{(t)}[S,S]\right) - \log\det\left(\Sigma[S,S]\right)\right| \leq U^{(t)} \lesssim \frac{d^3\sigma}{\sqrt{t}}\]
\end{lemma}
\begin{proof}{Proof.}

Because $\Sigma$ is a positive definite matrix, its minimum eigenvalue gives a lower bound for the eigenvalues of all its submatrices, so we have $\left\|\Sigma[S,S]^{-1}\right\| \leq \frac{1}{\lambda_{\min}}$ for all $\bm{s}$. 

By Lemma \ref{thm:matrixerror}, with a probability of at least $1-\delta$, we have
\[\max_{i,j}|\hat{\Sigma}^{(t)}[{i,j}] - \Sigma[{i,j}]| \leq \frac{8\lambda_{\max}}{3}\cdot \max\left\{d\sqrt{\frac{\ln(\pi^2 d^2 t^2/\delta)}{ct}}, \frac{d^2(\ln(\pi^2 d^2 t^2/\delta))}{ct}\right\},\]
which implies that 
\[3d\|\Sigma^{-1}\|\|\hat{\Sigma}^{(t)} - \Sigma\|\leq \frac{3d}{\lambda_{\min}}\cdot d \max_{i,j}|\hat{\Sigma}^{(t)} - \Sigma| \leq U^{(t)}\leq 1,\]
where second inequality is because the operator norm of a matrix can be upper bounded by its Frobenius norm.

Now, since we have checked the condition for Lemma \ref{thm:logdet}, we get 
\[\left|\log \det \left(\hat{\Sigma}^{(t)}[S,S]\right) - \log \det(\Sigma[S,S])\right| \leq 3d\|\Sigma^{-1}[S,S]\|\|\hat{\Sigma}^{(t)}[S,S] - \Sigma[S,S]\| \leq U^{(t)}.\]
Recall that 
\begin{equation}
    U^{(t)} = \frac{8\lambda_{\max}}{\lambda_{\min}}\cdot \max\left\{d^3\sqrt{\frac{\ln(\pi^2 d^2 t^2/\delta)}{ct}}, \frac{d^4(\ln(\pi^2 d^2 t^2/\delta))}{ct}\right\}.
\end{equation}
For sufficiently large $t$, the first term in the max operator dominates, which ends the proof.

     \hfill $\square$  \end{proof}

\begin{theorem}[Theorem \ref{thm:itr} in the main body]
    With a probability of at least $1-\delta$, for sufficiently large $T$, the cumulative regret of Algorithm \ref{alg:itr} is at most
    \(\Tilde{O}(d^3\sigma\sqrt{T})\).
\end{theorem}
\begin{proof}{Proof.}
    We first note that because $U^{(t)}$ is at most $O(\frac{d^3\sigma}{\sqrt{t}})$, it takes at most $O(d^6\sigma^2)$ steps to satisfy the condition in Lemma \ref{thm:simple}. Since there is no elimination happening before $U^{(t)}\leq 1$, there is at most $O(d^6\sigma^2)$ regret incured for this stage.

    After $U^{(t)}\leq 1$, according to Lemma \ref{thm:simple}, with the given high probability $1-\delta$, we can see that the optimal solution is never eliminated from the candidate set $\mathcal{S}$ by the choice of confidence bound width. 

    Moreover, for any $\bm{s}$ that still remains in the set $\mathcal{S}$, its gap to optimality is at most $\Tilde{O}(\frac{d^3\sigma}{\sqrt{t}})$. As such, the cumulative regret is less than \(\Tilde{O}(d^3\sigma\sqrt{T})\).
     \hfill $\square$  \end{proof}
\end{APPENDICES}

\newpage
\bibliographystyle{informs2014} 
\bibliography{references} 

\begin{thebibliography}{31}
\providecommand{\natexlab}[1]{#1}
\providecommand{\url}[1]{\texttt{#1}}
\providecommand{\urlprefix}{URL }

\bibitem[{Ashtiani et~al.(2020)Ashtiani, Ben-David, Harvey, Liaw, Mehrabian, \protect\BIBand{} Plan}]{ashtiani2020near}
Ashtiani H, Ben-David S, Harvey NJ, Liaw C, Mehrabian A, Plan Y (2020) Near-optimal sample complexity bounds for robust learning of gaussian mixtures via compression schemes. \emph{Journal of the ACM (JACM)} 67(6):1--42.

\bibitem[{Azar et~al.(2017)Azar, Osband, \protect\BIBand{} Munos}]{azar2017minimax}
Azar MG, Osband I, Munos R (2017) Minimax regret bounds for reinforcement learning. \emph{International conference on machine learning}, 263--272 (PMLR).

\bibitem[{Bibaut et~al.(2022)Bibaut, Kallus, \protect\BIBand{} Lindon}]{bibaut2022near}
Bibaut A, Kallus N, Lindon M (2022) Near-optimal non-parametric sequential tests and confidence sequences with possibly dependent observations. \emph{arXiv preprint arXiv:2212.14411} .

\bibitem[{Budynas et~al.(2011)Budynas, Nisbett et~al.}]{budynas2011shigley}
Budynas RG, Nisbett JK, et~al. (2011) \emph{Shigley's mechanical engineering design}, volume~9 (McGraw-Hill New York).

\bibitem[{Cesa-Bianchi \protect\BIBand{} Lugosi(2006)}]{CesaBianchiLugosi2006}
Cesa-Bianchi N, Lugosi G (2006) \emph{Prediction, Learning, and Games} (Cambridge, UK: Cambridge University Press), ISBN 978-0-521-85128-4.

\bibitem[{Cesa-Bianchi et~al.(2006)Cesa-Bianchi, Lugosi, \protect\BIBand{} Stoltz}]{cesa2006regret}
Cesa-Bianchi N, Lugosi G, Stoltz G (2006) Regret minimization under partial monitoring. \emph{Mathematics of Operations Research} 31(3):562--580.

\bibitem[{Chernoff(1959)}]{chernoff1959sequential}
Chernoff H (1959) Sequential design of experiments. \emph{The Annals of Mathematical Statistics} 30(3):755--770.

\bibitem[{Condon et~al.(2009)Condon, Deshpande, Hellerstein, \protect\BIBand{} Wu}]{condon2009algorithms}
Condon A, Deshpande A, Hellerstein L, Wu N (2009) Algorithms for distributional and adversarial pipelined filter ordering problems. \emph{ACM Transactions on Algorithms (TALG)} 5(2):1--34.

\bibitem[{Cover \protect\BIBand{} Thomas(2006)}]{cover-thomas2006}
Cover TM, Thomas JA (2006) \emph{Elements of Information Theory (Wiley Series in Telecommunications and Signal Processing)} (USA: Wiley-Interscience), ISBN 0471241954.

\bibitem[{Domingues et~al.(2021)Domingues, M{\'e}nard, Kaufmann, \protect\BIBand{} Valko}]{domingues2021episodic}
Domingues OD, M{\'e}nard P, Kaufmann E, Valko M (2021) Episodic reinforcement learning in finite mdps: Minimax lower bounds revisited. \emph{Algorithmic Learning Theory}, 578--598 (PMLR).

\bibitem[{Fampa \protect\BIBand{} Lee(2022)}]{fampa2022maximum}
Fampa M, Lee J (2022) \emph{Maximum-entropy sampling: algorithms and application} (Springer Nature).

\bibitem[{Gamala et~al.(2018)Gamala, Linn-Rasker, Nix, Heggelman, Van~Laar, Pasker-de Jong, Jacobs, \protect\BIBand{} Klaasen}]{gamala2018gouty}
Gamala M, Linn-Rasker S, Nix M, Heggelman B, Van~Laar J, Pasker-de Jong P, Jacobs J, Klaasen R (2018) Gouty arthritis: decision-making following dual-energy ct scan in clinical practice, a retrospective analysis. \emph{Clinical Rheumatology} 37:1879--1884.

\bibitem[{Gan et~al.(2021)Gan, Jia, \protect\BIBand{} Li}]{gan2021greedy}
Gan K, Jia S, Li A (2021) Greedy approximation algorithms for active sequential hypothesis testing. \emph{Advances in Neural Information Processing Systems} 34:5012--5024.

\bibitem[{Hellerstein et~al.(2017)Hellerstein, {\"O}zkan, \protect\BIBand{} Sellie}]{hellerstein2017max}
Hellerstein L, {\"O}zkan {\"O}, Sellie L (2017) Max-throughput for (conservative) k-of-n testing. \emph{Algorithmica} 77(2):595--618.

\bibitem[{Howard \protect\BIBand{} Ramdas(2022)}]{steven2022sequential}
Howard SR, Ramdas A (2022) {Sequential estimation of quantiles with applications to A/B testing and best-arm identification}. \emph{Bernoulli} 28(3):1704 -- 1728, \urlprefix\url{http://dx.doi.org/10.3150/21-BEJ1388}.

\bibitem[{Howard et~al.(2021)Howard, Ramdas, McAuliffe, \protect\BIBand{} Sekhon}]{steven2021time-uniform}
Howard SR, Ramdas A, McAuliffe J, Sekhon J (2021) {Time-uniform, nonparametric, nonasymptotic confidence sequences}. \emph{The Annals of Statistics} 49(2):1055 -- 1080, \urlprefix\url{http://dx.doi.org/10.1214/20-AOS1991}.

\bibitem[{Ipsen \protect\BIBand{} Rehman(2008)}]{ipsen2008perturbation}
Ipsen IC, Rehman R (2008) Perturbation bounds for determinants and characteristic polynomials. \emph{SIAM Journal on Matrix Analysis and Applications} 30(2):762--776.

\bibitem[{Khuri(2009)}]{khuri2009linear}
Khuri AI (2009) \emph{Linear model methodology} (Chapman and Hall/CRC).

\bibitem[{Kodialam(2001)}]{kodialam2001throughput}
Kodialam MS (2001) The throughput of sequential testing. \emph{International Conference on Integer Programming and Combinatorial Optimization}, 280--292 (Springer).

\bibitem[{Lattimore \protect\BIBand{} Szepesv{\'a}ri(2020)}]{lattimore2020bandit}
Lattimore T, Szepesv{\'a}ri C (2020) \emph{Bandit algorithms} (Cambridge University Press).

\bibitem[{Little \protect\BIBand{} Rubin(2019)}]{little2019statistical}
Little RJA, Rubin DB (2019) \emph{Statistical Analysis with Missing Data} (Hoboken, NJ: John Wiley \& Sons), 3 edition, ISBN 978-1-119-45724-2.

\bibitem[{Naghshvar \protect\BIBand{} Javidi(2013)}]{naghshvar2013active}
Naghshvar M, Javidi T (2013) {Active sequential hypothesis testing}. \emph{The Annals of Statistics} 41(6):2703 -- 2738, \urlprefix\url{http://dx.doi.org/10.1214/13-AOS1144}.

\bibitem[{Peng et~al.(2018)Peng, Tang, Lin, \protect\BIBand{} Chang}]{peng2018}
Peng YS, Tang KF, Lin HT, Chang E (2018) Refuel: Exploring sparse features in deep reinforcement learning for fast disease diagnosis. \emph{Advances in Neural Information Processing Systems}, volume~31 (Curran Associates, Inc.).

\bibitem[{Segev \protect\BIBand{} Shaposhnik(2022)}]{segev2022poly}
Segev D, Shaposhnik Y (2022) A polynomial-time approximation scheme for sequential batch testing of series systems. \emph{Operations Research} 70(2):1153--1165, \urlprefix\url{http://dx.doi.org/10.1287/opre.2019.1967}.

\bibitem[{Vershynin(2018)}]{vershynin2018high}
Vershynin R (2018) \emph{High-dimensional probability: An introduction with applications in data science}, volume~47 (Cambridge university press).

\bibitem[{Wald \protect\BIBand{} Wolfowitz(1948)}]{wald1948optimum}
Wald A, Wolfowitz J (1948) Optimum character of the sequential probability ratio test. \emph{The Annals of Mathematical Statistics} 326--339.

\bibitem[{Weissman et~al.(2003)Weissman, Ordentlich, Seroussi, Verdu, \protect\BIBand{} Weinberger}]{weissman2003inequalities}
Weissman T, Ordentlich E, Seroussi G, Verdu S, Weinberger MJ (2003) Inequalities for the l1 deviation of the empirical distribution. \emph{Hewlett-Packard Labs, Tech. Rep} 125.

\bibitem[{Yu et~al.(2023)Yu, Li, Kim, Huang, Luo, \protect\BIBand{} Wang}]{yu2023deep}
Yu Z, Li Y, Kim JC, Huang K, Luo Y, Wang M (2023) Deep reinforcement learning for cost-effective medical diagnosis. \emph{The Eleventh International Conference on Learning Representations}.

\bibitem[{Zanette \protect\BIBand{} Brunskill(2019)}]{zanette2019tighter}
Zanette A, Brunskill E (2019) Tighter problem-dependent regret bounds in reinforcement learning without domain knowledge using value function bounds. \emph{International Conference on Machine Learning}, 7304--7312 (PMLR).

\bibitem[{Zhang et~al.(2006)Zhang, Doherty, Bardin, Pascual, Barskova, Conaghan, Gerster, Jacobs, Leeb, Liot{\'e} et~al.}]{zhang2006eular}
Zhang W, Doherty M, Bardin T, Pascual E, Barskova V, Conaghan P, Gerster J, Jacobs J, Leeb B, Liot{\'e} F, et~al. (2006) Eular evidence based recommendations for gout. part ii: Management. report of a task force of the eular standing committee for international clinical studies including therapeutics (escisit). \emph{Annals of the rheumatic diseases} 65(10):1312--1324.

\bibitem[{Zidek et~al.(2002)Zidek, Sun, \protect\BIBand{} Le}]{zidek2002designing}
Zidek JV, Sun W, Le ND (2002) Designing and integrating composite networks for monitoring multivariate gaussian pollution fields. \emph{Journal of the Royal Statistical Society Series C: Applied Statistics} 49(1):63--79, ISSN 0035-9254, \urlprefix\url{http://dx.doi.org/10.1111/1467-9876.00179}.

\end{thebibliography}

\end{document}